\documentclass{article}

\PassOptionsToPackage{numbers, compress}{natbib}
\usepackage[preprint]{neurips_2023}

\usepackage[utf8]{inputenc} 
\usepackage[T1]{fontenc}    
\usepackage{hyperref}       
\usepackage{url}            
\usepackage{booktabs}       
\usepackage{amsfonts}       
\usepackage{nicefrac}       
\usepackage{microtype}      
\usepackage{xcolor}         
\usepackage[pdftex]{graphicx}
\usepackage{amssymb}
\usepackage{amsthm}
\usepackage{amsmath}
\newtheorem{prop}{Proposition}

\usepackage{algorithm} 
\usepackage{algpseudocode} 
\usepackage{threeparttable} 
\usepackage{makecell} 
\usepackage{multirow} 
\usepackage{soul} 
\usepackage{float}
\usepackage{subfig}
\usepackage{wrapfig}
\usepackage{authblk} 
\usepackage[defaultcolor=black]{changes}
\setdeletedmarkup{}
\newtheorem{corollary}{Corollary}

\title{
HG2P: Hippocampus-inspired High-reward Graph and Model-Free Q-Gradient Penalty for Path Planning and Motion Control
}

\author[1]{
  Haoran Wang
}
\author[ ,1]{
  Yaoru Sun\thanks{Corresponding author: Yaoru Sun.}
}
\author[1]{
  Zeshen~Tang
}
\author[2]{
  Haibo~Shi
 }
 \author[1]{
  Chenyuan~Jiao
 }
 \affil[1]{Department of Computer Science and Technology, Tongji University\protect\\ \texttt{\{1910664, yaoru, 2011610, 2432024\}@tongji.edu.cn}}
  \affil[2]{School of Statistics and Management, Shanghai University of Finance and Economics\protect\\ \texttt{shihaibo@sufe.edu.cn}}

\begin{document}

\maketitle

\begin{abstract}
Goal-conditioned hierarchical reinforcement learning (HRL) decomposes complex reaching tasks into a sequence of simple subgoal-conditioned tasks, showing significant promise for addressing long-horizon planning in large-scale environments. This paper bridges the goal-conditioned HRL based on graph-based planning to brain mechanisms, proposing a hippocampus-striatum-like dual-controller hypothesis. Inspired by the brain mechanisms of organisms (i.e., the high-reward preferences observed in hippocampal replay) and instance-based theory, we propose a high-return sampling strategy for constructing memory graphs, improving sample efficiency. Additionally, we derive a model-free lower-level Q-function gradient penalty to resolve the model dependency issues present in prior work, improving the generalization of Lipschitz constraints in applications. Finally, we integrate these two extensions, High-reward Graph and model-free Gradient Penalty (HG2P), into the state-of-the-art framework ACLG, proposing a novel goal-conditioned HRL framework, HG2P+ACLG\footnote{Code is available at \url{https://github.com/HaoranWang-TJ/HG2P_ACLG_official}.}. Experimentally, the results demonstrate that our method outperforms state-of-the-art goal-conditioned HRL algorithms on a variety of long-horizon navigation tasks and robotic manipulation tasks.

\end{abstract}

\section{Introduction}
\label{sec:introduction}
Goal-reaching and navigation tasks are common problems in many sequential decision-making and control tasks. For example, robotic assembly tasks can generally be decomposed into three subtasks: fetching an object, navigating to a specified position, and then putting it over there. Goal-conditioned hierarchical reinforcement learning (HRL) aims to solve this problem by breaking the final target into a series of subgoals. Recent successes with goal-conditioned HRL to complex and long-horizon tasks have been driven by well-designed subgoal-generating mechanisms, containing the $k$-step adjacency constraint \cite{zhang2020generating, zhang2022adjacency}, contrastive triplet loss-based subgoal representations \cite{li2020learning, li2021active, zhang2021world}, and graph-based planning \cite{huang2019mapping, kim2021landmark, lee2022dhrl, kim2023imitating}.
\added{Among these, motion planning, particularly graph-based approaches, plays a crucial role in enabling autonomous agents to accomplish missions swiftly and efficiently \cite{chai2024cooperative}.}
\added{For instance, Chai et al. \cite{chai2022design} proposed a hierarchical control framework to control a mobile robot, where an upper-level motion planning network generates optimized planned trajectories. At the lower level, an RL-based controller is developed to accomplish the waypoint tracking task.}
\added{Furthermore, by equipping the high-level planner with trajectory optimization methods that integrate constraints, including boundary conditions and collision-avoidance constraints, the hierarchical framework can be further enhanced for practical deployment in real-world scenarios, such as first-aid supply delivery \cite{chai2024two}, vehicle parking \cite{chai2022deep}, and overtaking in fail-safe situations \cite{chai2022multiphase}.}

Notably, recent studies in psychology and neuroscience reveal that integrating graph-based motion planning into model-free RL may potentially be associated with brain mechanisms in decision-making and memory \cite{gershman2017reinforcement, lin2018episodic}. For example, representations of place cells and grid cells in the rodent hippocampus can be viewed as the cognitive maps for spatial memory and navigation \cite{moser2015place, killian2018grid}. Inspired by the mechanisms of the brain, recent advances \cite{kang2023sample, zhu2019episodic} in episodic RL support the idea that experience should be stored in the form of the graph structure for more efficient recall and reconstruction, imitating the netlike connected engrams in the human brain. In the memory graph, each node corresponds to a visited state, each edge corresponds to an action, with the state-action value attributed to the edge. Under near-deterministic settings, researchers \cite{kang2023sample} focused on updating the edge attributes along the valuable paths to reinforce memory, thereby improving sample efficiency. Recent neuroimaging findings \cite{singer2009rewarded, ambrose2016reverse, murty2017selectivity, michon2019post, michon2021single, elliott2020neural} have demonstrated the underlying rationale for selective memory enhancement. Specifically, hippocampal replay activity, particularly the reverse replay during rest, is biased toward highly rewarded places, where high reward value enhances memory retention by selectively enhancing the reactivation and accelerating reorganization of the hippocampal spatial map. Additionally, novel experiences are also known to induce and increase replay activity, enhancing the acquisition and consolidation of new information. The effects of novelty or fear are tied to emotional (up or down) states, modulated by the noradrenergic system \cite{takeuchi2016locus}.

\begin{figure*}[htbp]
\centering
\includegraphics[width=0.88\textwidth]{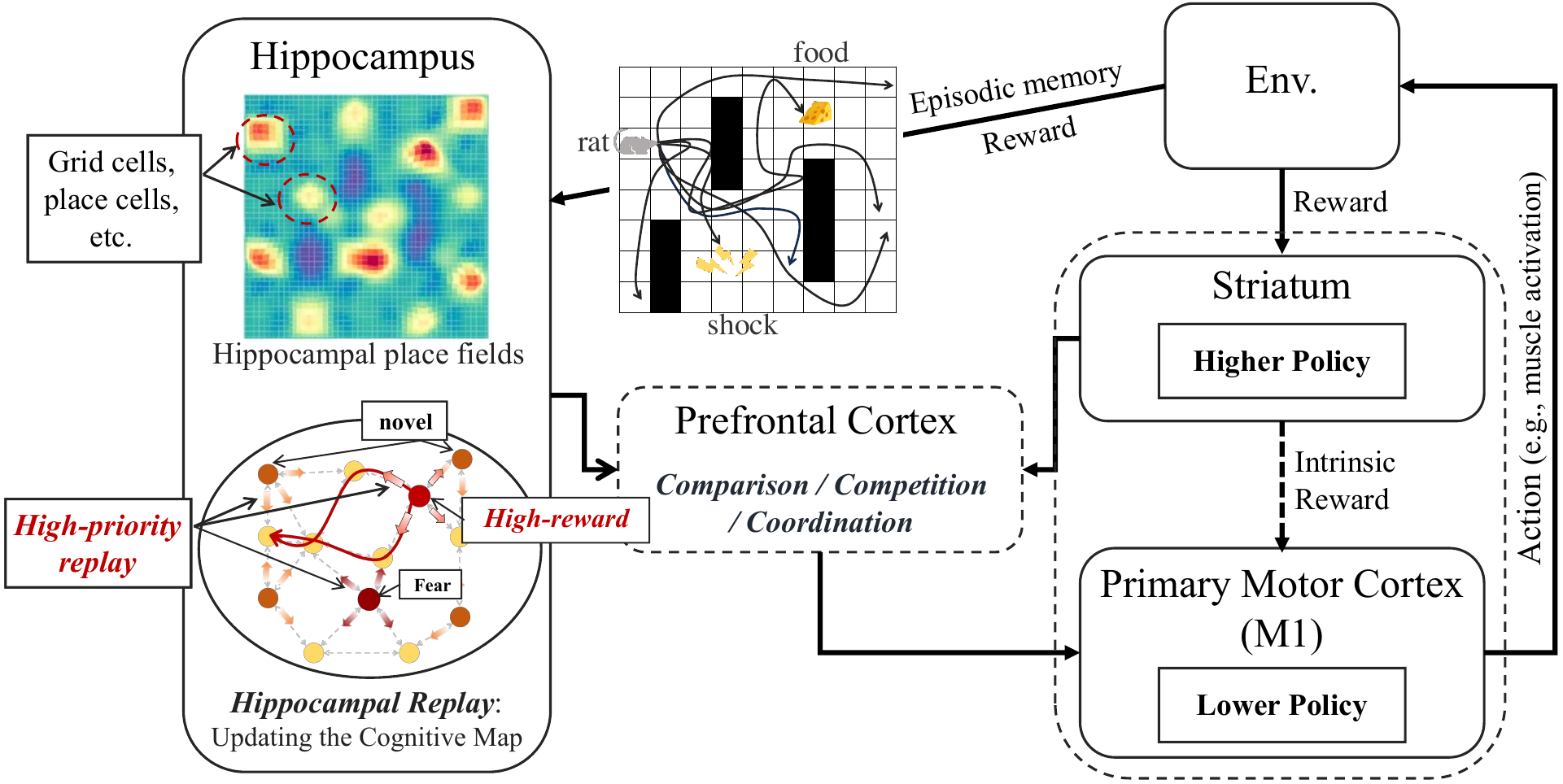}
\caption{Schematic of the proposed hippocampus-striatum-like dual-controller hypothesis. \replaced{Chersi et al. \cite{chersi2015cognitive} proposed a minimal cognitive architecture for spatial navigation, comprising two principal mechanisms: (i) the hippocampus, which encodes environmental locations to support goal-directed decision making; and (ii) the striatum, which learns stimulus-response associations. Here, we further elaborate on the functional roles of the hippocampus and striatum modules by incorporating hippocampal replay and the involvement of the primary motor cortex (M1).}{The hypothesis builds upon the minimal cognitive architecture for spatial navigation proposed by Chersi et al \cite{chersi2015cognitive}.  We expand this architecture by detailing the roles of the hippocampus and striatum modules: (a) The hippocampus models place fields, similar to landmarks, based on place cells and grid cells. During training, hippocampal replay, particularly the reverse replay, prioritizes memories associated with high rewards \cite{singer2009rewarded, ambrose2016reverse, murty2017selectivity, michon2019post, michon2021single, elliott2020neural}, strong emotional responses, and novel experiences \cite{takeuchi2016locus}; (b) The striatum functions as the higher-level policy in HRL. Furthermore, we detail the output pathway from the prefrontal cortex: motor signals are modulated by the primary motor cortex (M1), functioning like the lower-level policy in HRL, and are then transmitted along the spinal cord to the end effectors for motor control.}}
\label{hippocampus-striatum-like}
\end{figure*}
Based on existing neurobiological studies and findings, we can effectively link the existing goal-conditioned HRL with graph-based planning to brain mechanisms, proposing a dual-controller hypothesis. Indeed, Chersi et al. \cite{chersi2015cognitive} were the first to propose a minimal cognitive architecture for spatial navigation, a minimal circuit composed of the hippocampus and striatum. Additionally, recent studies \cite{lin2018episodic, zhu2019episodic} have argued that two particular controllers should be integrated: one associated with the hippocampus, which provides the cognitive map for goal-directed decision-making, and the other associated with the ventral striatum, which learns stimulus-response associations. Inspired by this \cite{chersi2015cognitive}, our dual-controller hypothesis based on the hippocampus-striatum system is illustrated in Fig.~\ref{hippocampus-striatum-like}.
\added{The hypothesis builds upon the minimal cognitive architecture for spatial navigation proposed by Chersi et al. \cite{chersi2015cognitive}.  We expand this architecture by detailing the roles of the hippocampus and striatum modules: (a) The hippocampus models place fields, similar to landmarks, based on place cells and grid cells. During training, hippocampal replay, particularly the reverse replay, prioritizes memories associated with high rewards \cite{singer2009rewarded, ambrose2016reverse, murty2017selectivity, michon2019post, michon2021single, elliott2020neural}, strong emotional responses, and novel experiences \cite{takeuchi2016locus}; (b) The striatum functions as the higher-level policy in HRL. Furthermore, we detail the output pathway from the prefrontal cortex: motor signals are modulated by the primary motor cortex (M1), functioning like the lower-level policy in HRL, and are then transmitted along the spinal cord to the end effectors for motor control.}

\begin{figure*}[h]
\centering
\includegraphics[width=0.99\textwidth]{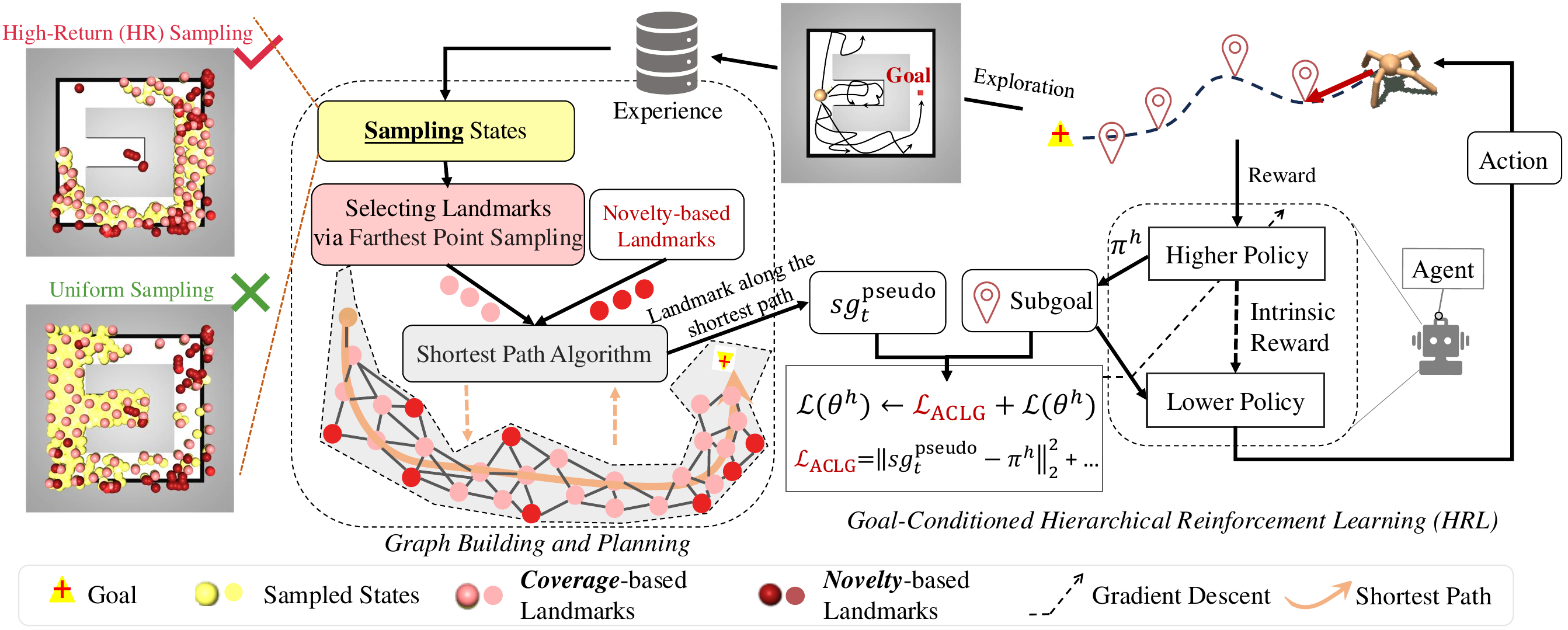}
\caption{\added{Overview of the proposed framework.} \added{The \textbf{right} part} illustrates the architecture of goal-conditioned hierarchical RL, where the higher-level policy observes the environment and generates a high-level action (i.e., subgoal), while the lower-level policy works to reach the assigned subgoal. \added{The \textbf{left} part details the graph‑building and planning process: coverage‑ and novelty‑based landmarks are selected to construct a topological map, from which the most urgent landmark is chosen as the next desired subgoal.} \added{The \textbf{left-most} part} visually compares our high-return (HR) sampling with the commonly used uniform sampling in landmark-based planning. Remarkably, our graph contracts rapidly upon encountering high-value regions, similar to \added{the behavior observed in} the \textit{slime mold maze} experiment \cite{nakagaki2000maze}. \added{The maze environment depicted in this diagram is discussed in detail in \ref{embossed_u_maze}.}}
\label{fig:maze}
\end{figure*}
\replaced{Based on this hypothesis, we instantiate the hippocampus-striatum-inspired dual-controller as a novel goal-conditioned hierarchical RL framework, equipped with a high-level, dynamically constructed landmark-based map, as illustrated in Fig.~\ref{fig:maze}. W}{Based on such a hypothesis, w}e argue that the following studies are biologically plausible for accelerated learning: HIGL \cite{kim2021landmark} proposed novelty-based sampling; ACLG \cite{wang2024guided} decoupled the loss terms of HIGL to enable direct competition between the graph-based (hippocampus-like) controller and the model-free (striatum-like) controller. However, previous studies on graph-based planning neglect \textit{the high-reward preferences observed in hippocampal replay}. To address this, we propose adopting a high-return sampling strategy to replace uniform sampling when constructing memory graphs, as shown in \added{the left-most part of} Fig.~\ref{fig:maze}. The high-return sampling strategy will assign larger weights to high-return experiences. This is also inspired in part by instance-based decision theory, i.e., the performance of the target policy is anchored on the performance of the behavior policy. Therefore, learning from high-performing policies’ trajectories holds the potential to achieve greater policy improvement over the behavior policy.

Meanwhile, we note that Wang et al. \cite{wang2024guided} derived a model-inferred upper bound of Q-function's Lipschitz constant w.r.t. actions, aiming at suppressing sharp lower-level Q-function gradients. The underlying reason is that the lower-level policy interacts directly with the environment, and reinforcing its robustness can prevent catastrophic failures. Also, Lipschitz's continuous constraints of lower-level Q-function make the goal-reaching distance estimation based on value function \cite{huang2019mapping} more reliable, avoiding overestimation for out-of-distribution (OOD) states. Nevertheless, approximating the environment dynamics to derive the upper bound is particularly costly and impractical for complex dynamical environments with high-dimensional observations. In this paper, we strategically propose a \textit{model-free} Q-gradient penalty by deriving the upper bound of the lower-level Q-function's Lipschitz constant \textit{w.r.t. input states and subgoals}. Interestingly, this upper bound is associated with the higher-level policy, promoting inter-level coordination. 

All in all, in this paper, we propose a novel goal-conditioned HRL framework, which primarily introduces two crucial extensions to existing frameworks: 1) high-return sampling for graph construction, simulating the high-reward preferences observed in hippocampal replay; 2) a model-free Q-gradient penalty, providing an alternative approach to model-based penalty \cite{wang2024guided}. The first extension relates to the Hippocampus-inspired \textbf{H}igh-reward \textbf{G}raph, and the second to model-free \textbf{G}radient \textbf{P}enalty, so the proposed method is called \textbf{HG2P}.
We integrate these extensions into a state-of-the-art framework, ACLG \cite{wang2024guided}, and evaluate the effectiveness of the novel framework on a set of complex and long-horizon tasks, including both maze-related navigation tasks and robot manipulation tasks. The experimental results demonstrate that the proposed framework contributes to improved training stability and higher sample efficiency. The main contributions of this work are summarized as follows:
\begin{itemize}
    \item It is the first attempt to interpret landmark-based planning through the hippocampus-striatum-like dual-controller hypothesis.
    \item Inspired by such hypothesis and instance-based theory, we propose a high-return (HR) sampling strategy that emphasizes the importance of high-return episodic memory in constructing memory graphs, thereby enhancing sample efficiency.
    \item To eliminate dependence on models, we propose a model-free Q-gradient penalty (MF-GP) module that serves as an alternative to the model-based Q-gradient penalty used in prior work \cite{wang2024guided}.
\end{itemize}

The remaining part of this paper is organized as follows:
Section~\ref{sec:preliminaries} introduces the preliminaries, including HRL, adjacency constraint, landmark-based planning, and the prior state-of-the-art approaches, i.e., HIGL and its variant ACLG. Section~\ref{sec:rw} briefly introduces the related works, focusing on sampling in episodic memory and enforcing local Lipschitzness. Section~\ref{sec:method} describes two crucial components of the proposed framework, providing a detailed explanation and formulation. In Section~\ref{sec:experiments}, we describe the experimental environments, conduct ablation studies to explore the impact of different parameters, present the main results, and discuss the principal findings. Section~\ref{sec:cfw} summarizes the concluding remarks of this article, while Section~\ref{sec:limit_cfw} discusses several limitations and outlines potential directions for future research.

\section{Preliminaries}
\label{sec:preliminaries}

We formulate a continuous control task with a finite-horizon, goal-conditioned Markov decision process (MDP) defined as a tuple $\left(\mathcal{S}, \mathcal{G}, \mathcal{A}, \mathcal{P}, \mathcal{R}, \gamma \right)$, where $\mathcal{S}$ is the state space, $\mathcal{G}$ is the goal space, $\mathcal{A}$ denotes the action space, $\mathcal{P}: \mathcal{S} \times \mathcal{A} \rightarrow \mathcal{S}$ is the transition function defining the transition dynamics of environment, $\mathcal{R}: \mathcal{S} \times \mathcal{A} \times \mathcal{G} \rightarrow \mathbb{R}$ is the reward function, and $\gamma \in \left[0, 1\right)$ is a discount factor. To be specific, when the environment  takes an action $a_t \in \mathcal{A}$, it will transition from $s_t \in \mathcal{S}$ to a new state $s_{t+1} \in \mathcal{S}$ while yielding a reward $R_t \in \mathcal{R}$, where $s_{t+1} \sim \mathcal{P}\left(s_{t+1}|s_t, a_{t}\right)$ and $R_t$ is conditioned on a goal $g \in \mathcal{G}$. In many real-world scenarios, goal-conditioned control tasks pose a significant challenge for general reinforcement learning, particularly in long-horizon tasks where the goal is far away. Following prior work \cite{wang2024guided}, we presented a goal-conditioned HRL framework to deal with these challenging long-horizon tasks. In our goal-conditioned HRL setup, the hierarchical framework typically has two layers: higher- and lower-level policies. The higher-level policy observes the current state $s_t$ of the environment and produces a high-level action, i.e., a subgoal $sg_t \in \mathcal{G}$, every $h$ steps. The subgoal indicates a desired change of states, such as a relative or absolute change of the location. Meanwhile, the lower-level policy attempts to reach the assigned subgoal within a $h-$step interval. We assume that the higher- and lower-level policies are parameterized by $\theta^h$ and $\theta^l$, respectively. The procedure mentioned above can be formulated as follows: $sg_t = a^h_t \sim \pi_{\theta^h} (s_t, g)$ when $t \equiv 0$ (mod $h$); $a^l_t \sim \pi_{\theta^l} (s_t, sg_t)$. Here, the higher-level action $a^h_t \in \mathcal{G}$ represents a subgoal, which has a different meaning from the low-level atomic actions $a^l_t \in \mathcal{A}$. When $t \not\equiv 0$ (mod $h$) and the subgoal indicates an absolute location, $sg_{t+1} = sg_t$; otherwise, the subgoal will automatically evolve following a pre-defined subgoal transition function: $sg_t = H(s_{t-1}, sg_{t-1}, s_t)$. Generally, the subgoal transition function is defined in the goal space. Here, we first introduce a known mapping function that transforms a state into the goal space: $\varphi: \mathcal{S} \rightarrow \mathcal{G}$. Then, the subgoal transition function can be expressed as:  $sg_t = H(s_{t-1}, sg_{t-1}, s_t)=sg_{t-1}+\varphi(s_{t} - s_{t-1})$. The self-transition of subgoals enables the lower-level agent to behave like an autonomous dynamical system within the $h-$step interval.

Furthermore, during interaction with the environment, the higher-level agent is motivated by the external rewards, i.e., receiving the entire feedback by accumulating all environmental rewards within $h$-step horizon: $r^h_t=\sum_{i=t}^{t+h-1} R_i$; while the lower-level agent is intrinsically motivated by the internal reward: $r^l_t=-\lVert sg_{t+1} - \eta \varphi(s_{t+1}) \rVert_2$, where $\eta$ denotes a Boolean hyperparameter whose value is 0/1 for the relative/absolute subgoal scheme. The higher- and lower-level agents aim to maximize the expected discounted rewards $\mathbb{E}_{L\in \left \{ l, h\right \} } \left[ \sum^{\infty}_{t=0} \gamma^t r^L_t\right]$. To simplify notation, hereinafter, we use superscripts $L\in \left \{ l, h\right \}$ to distinguish the aforementioned variables involved in the higher-level and lower-level design. For example, $a^l$ and $a^h$ indicate the lower-level and higher-level actions, respectively. In hierarchical reinforcement learning, the objective is to find the optimal behavior policies $\pi^L$, with parameters $\theta^L$, which maximize such discounted rewards. In standard actor-critic
methods, the parametrized policies $\pi^L$ can be updated through the deterministic policy gradient algorithm:
\begin{equation}
\mathcal{L}(\theta^L) = \mathbb{E}_{o^L \sim \mathcal{D}_{L}} \left[\nabla_{a^L}Q \left(o^L, a^L\right)\Big|_{a^L=\pi^L(o^L)} \nabla_{\theta^L}\pi^L \right]
\end{equation}
Here, we use $o^L$ instead of $s^L$, because at the lower-level, the observation includes the state $s_t$ and the subgoal $sg_t$, i.e., $o^l_t=\langle s_t, sg_t \rangle$; while at the higher-level, it includes $s_t$ and the goal $g$, i.e., $o^h_t=\langle s_t, g \rangle$. Q$ \left(o^L_t, a^L_t\right)=\mathbb{E}_{o^L_i \sim \mathcal{P}, a^L_i \sim \pi^L} \left[ \sum^{\infty}_{i=t} \gamma^i r^L_i \big| o^L_t, a^L_t\right]$, is known as the value function, with estimating the expected return when performing action $a^L_t$ and the subsequent actions induced by $\pi^L$. Generally, the value function can be approximated with a
differentiable function approximator, with parameters $\phi^L$, and can be learned using temporal difference learning:
\begin{equation}
\mathcal{L}(\phi^L) = \mathbb{E}_{(o^L_t, a^L_t, r^L_t, o^L_{t+1}) \sim \mathcal{D}_{L}} \left[ \Big( Q \left(o^L_t, a^L_t\right)- \left(r^L_t + \gamma Q \left(o^L_{t+1}, \pi^L(o^L_{t+1})\right)\right) \Big)^2\right]
\end{equation}
In practice, we utilize the Twin Delayed DDPG (TD3) \cite{fujimoto2018addressing} algorithm, a variant of the popular DDPG \cite{timothy2016continuous} algorithm, for instantiating the lower- and higher-level agents. The TD3 can substantially improve stability and performance by introducing three critical tricks. For more details, please refer to the literature \cite{fujimoto2018addressing}.

\subsection{Adjacency Constraint}
Given the hierarchical framework described above, innumerable subgoals can induce behavioral policies to achieve the same rollout. Learning such a large subgoal representation space poses difficulties for training efficiency. Zhang et al. \cite{zhang2020generating, zhang2022adjacency} restricted the high-level action space, i.e., the subgoal space, to a $k$-step adjacent region. The work maintained a binary $k$-step adjacency matrix to memorize whether two states are $k$-step adjacent. To address the non-differentiable and generalization issues, in practice, the authors parameterized the $k$-step adjacency region with a neural network. Let $\psi$ be an adjacency network parameterized by $\Phi$. Subsequently, the adjacency information stored in $k$-step adjacency matrix can be further distilled into the adjacency network by minimizing the following contrastive-like loss: $\mathcal{L}_{{\rm adj}}(\Phi)=\mathbb{E}_{s_i,s_j\in\mathcal{S}}[l\cdot \max(||\psi_{\Phi}(\varphi(s_i))-\psi_{\Phi}(\varphi(s_j))||_2-\zeta_k, 0)+(1-l)\cdot \max(\zeta_k+\delta_{\rm adj}-||\psi_{\Phi}(\varphi(s_i))-\psi_{\Phi}(\varphi(s_j))||_2, 0)]$, where $\delta_{\rm adj} > 0 $ is a hyperparameter indicating a margin between embeddings, $\zeta_k$ is a scaling factor, and $l \in \{0,1\}$ is the label indicating whether $s_i$ and $s_j$ are $k$-step adjacent. When in use, the adjacency network can measure whether two states are $k$-step adjacent using the Euclidean distance. Specifically, the $k$-step adjacent estimation of two states $s_i$ and $s_j$ can be calculated as: $\max(||\psi_{\Phi}(\varphi(s_i))-\psi_{\Phi}(\varphi(s_j))||_2 - \zeta_k, 0)$, which means that the adjacency network will output a non-zero value when $||\psi_{\Phi}(\varphi(s_i))-\psi_{\Phi}(\varphi(s_j))||_2 > \zeta_k$, indicating non-adjacency.

\subsection{Landmark-Guided Planning}\label{Landmark_Sampling}
Recently, graph-based planning algorithms have been the preferred tool for endowing goal-based RL agents with the ability to navigate over complex long-horizon problems. Classical graph-based planning commonly contains two parts: (a) sampling landmarks from the visited state space and (b) building a graph planner to select waypoints.

\paragraph{(a) Landmark Sampling}
In order to build a landmark-based map to abstract the state space, a pool of states is typically sampled \textit{uniformly} from the replay buffer to cover the visited state space. However, planning over the directly sampled subset will impose a significant computational burden. Furthermore, researchers proposed to deploy the farthest point sampling (FPS) algorithm for further sparsification\cite{huang2019mapping}. These resulting sparse states are named \textit{landmarks}, a concept borrowed from navigation literature.

\paragraph{(b) Graph Building and Planning}
After obtaining a collection of landmarks, denoted as $\mathcal{S}_{LM}$, we turn our attention to building a directed weighted graph for high-level planning. To build the graph, we first connect any pair of landmarks and assign an estimated distance between them as the weight of the connecting edge. Following prior works\cite{huang2019mapping, kim2021landmark, wang2024guided}, the edge weights are formed by the UVFA-estimated distance, approximated as: $w_e(s_i,s_j) \leftarrow -\min_aQ(s_i, \varphi(s_j - \eta s_i),a; \phi^l)$, i.e., $-Q(s_i, \varphi(s_j - \eta s_i), \pi(s_i, \varphi(s_j - \eta s_i);\theta^l); \phi^l)$, $\forall s_i, s_j \in \mathcal{S}_{LM}$. As indicated, the distance of state-goal pairs is estimated by the lower-level value function. Since the distance estimation is inaccurate for pairs of states that are far apart, connections with edge weights exceeding a preset threshold will be removed. To this end, the shortest path search algorithm is run to choose the next subgoal, guiding the behavioral policy to reach the final goal. Specifically, the selected next landmark, the very first landmark in the shortest path from the current state $s_t$ to the final goal $g$, can be found by: $sg^{\text{plan}}_t = \arg\min_{\varphi(s_i)}[w_e(s_t, s_i) + w_e(s_i, g)]$, $\forall s_i \in \mathcal{S}_{LM}$.

\subsection{HIGL\cite{kim2021landmark} and its variant ACLG\cite{wang2024guided}}
Additionally, recent works have integrated the adjacency constraint with landmark planning to ensure both reachability and efficient exploration. Notably, Kim et al. \cite{kim2021landmark} proposed HIGL, which further improves exploration and decision-making by introducing two critical tricks: (1) the novelty-based sampling scheme that stores newly encountered states as additional landmarks; (2) restricting the action space of a high-level policy to a $k$-step adjacent region centered at a selected landmark. To be specific, in addition to randomly sampling from the replay buffer, HIGL also sampled novel states by calculating RND-based novelty of states \cite{burda2019exploration}, i.e., \textit{novelty-based landmarks} $\mathcal{S}_{LM}^{\rm nov}$. To differentiate, we refer to the previously mentioned uniformly sampled landmarks as \textit{coverage-based landmarks} $\mathcal{S}_{LM}^{\rm cov}$. Thus, the final set of landmarks is given by $\mathcal{S}_{LM} = \mathcal{S}_{LM}^{\rm cov} \cup \mathcal{S}_{LM}^{\rm nov}$. Then, after performing the graph search on such a landmark-based graph, the most urgent landmark was found and shifted towards the current state for reachability. The new pseudo-landmark is defined as follows: $sg^{\text{pseudo}}_t = sg^{\text{plan}}_t + \delta_{\text{pseudo}}\cdot\frac{sg^{\text{plan}}_t - \varphi(s_t)}{||sg^{\text{plan}}_t - \varphi(s_t)||_2}$, where $\delta_{\text{pseudo}}$ is the shift magnitude. Finally, HIGL incorporated the adjacency constraint and the landmark planning into the goal-conditioned HRL framework by introducing the following loss:
\begin{equation}
\mathcal{L}_{\rm HIGL}(\theta^h) = \lambda^{\rm HIGL}_{\rm LM} \cdot \max(||\psi_{\Phi}(sg_{t}^{\rm pseudo}) -\psi_{\Phi}(\pi(s_t, g;\theta^h))||_2 - \zeta_k, 0)
\end{equation}
\replaced{w}{W}here $\lambda^{\rm HIGL}_{\rm LM} \in \mathbb{R}^+$ is the balancing coefficient controlling the effect of the loss term $\mathcal{L}_{\rm HIGL}(\theta^h)$.

Although HIGL is highly efficient for solving goal-conditioned HRL tasks, the performance is not satisfactory because the entanglement between the adjacency constraint and landmark-based planning was not well balanced.
In a recent study\cite{wang2024guided}, researchers proposed a new integrated method, ACLG, by disentangling the two components. ACLG only made minor modifications to the aforementioned loss term of HIGL:
\begin{equation}
\begin{aligned}
\mathcal{L}_{\rm ACLG}&(\theta^h) \\= & \lambda_{\rm adj} \cdot \max(||\psi_{\Phi}(\varphi(s_t))-\psi_{\Phi}(\pi(s_t, g;\theta^h))||_{2} - \zeta_k, 0) \\ \qquad &+ \lambda^{\rm ACLG}_{\rm LM} \cdot ||sg_{t}^{\rm pseudo} - \pi(s_t, g;\theta^h) ||^2_2
\label{aclg_loss}
\end{aligned}
\end{equation}
\replaced{w}{W}here the hyperparameters $\lambda_{\rm adj} \in \mathbb{R}^+$ and $\lambda^{\rm ACLG}_{\rm LM} \in \mathbb{R}^+$ are introduced to better balance the adjacency constraint and landmark planning.

\section{Related work}
\label{sec:rw}
\subsection{Episodic Memory and Sampling in Reinforcement Learning}
For most offline RL algorithms, experience replay has been the fundamental component, where previously collected experiences are reused to improve sample efficiency. Naturally, the construction and utilization (sampling) strategies of the replay buffer are critical factors influencing sample efficiency. Model-based RL methods learn dynamics models directly from historical data and then predict upcoming states. However, in order to learn such dynamics models, model-based RL methods often take orders of magnitudes more data. From the perspective of improving sampling efficiency, episodic RL methods propose to use high-return episodic memory. Specifically, in a near-deterministic environment, an agent can use episodic memory to record experiences and later imitate sequences of actions that yield high rewards. The episodic strategy is in part inspired by the human brain, where the hippocampus and related medial temporal lobe structures play an integral role in episodic memory encoding and processing. Previous works on episodic RL used tabular-like representation to memorize past experience and retrieved the highly-rewarding sequences through context-based lookup \cite{blundell2016model, pritzel2017neural}. Furthermore, recent theoretical work \cite{lin2018episodic} argued that both model-free and model-based control systems should be considered together for better sample efficiency. Because, in the working mechanism of the human brain, two particular controllers complement each other: one associated with the hippocampus and model-based learning, and the other with the ventral striatum and dopamine. Therefore, in related literature \cite{lin2018episodic}, both model-free RL and episodic RL methods were leveraged to provide an inference target and a memory target for the agent. However, these studies merely utilized a tabular-like memory to store original data or episodic trajectories without considering the association between past experiences. Several extensions \cite{zhu2019episodic, kang2023sample} have been proposed to fully exploit the information in episodic memory by reconstructing memory graphs. The concept of associative memory was introduced in the form of a directed graph, where each visited state was encoded as a node attributed with a state-action value \cite{zhu2019episodic}. Since the same states may appear in multiple trajectories, these nodes will contribute to connecting different trajectories and facilitating inter-episode value propagation. In the end, similar to previous work \cite{lin2018episodic}, researchers used the graph-based state-action value as a memory target for the agent \cite{zhu2019episodic}. Another work introduced a similar concept termed "engram"\cite{kang2023sample}. The framework abstracted key nodes from the memory graph, utilized a shortest-path search algorithm to get valuable paths, and prioritized updating these nodes within the reconstructed paths. Our work differs from episodic RL methods, which often succeed under the strong assumption of a near-deterministic environment with a fixed goal. In contrast, our objective is to address more complex tasks that involve varying goals.

Another related line of work is about trajectory-wise sampling. From the perspective of instance-based decision theory, the performance of the target policy is anchored on the performance of the behavior policy that collects interaction data. Therefore, anchoring in a high-performing dataset favors the performance of the target policy; otherwise, it may hinder the performance. Imitation learning methods, such as decision transformers \cite{chen2021decision, furuta2021generalized}, tend to perform behavior cloning on only trajectories with higher cumulative rewards, ordered by episode returns. However, these studies use a simplistic approach by discarding low-return episodic experiences. The most closely related prior work to our method is that of Hong et al. \cite{hong2022harnessing}, which employs trajectory-wise reweighting to fully exploit rare high-performing data. However, from an application perspective, Hong et al. \cite{hong2022harnessing} use weighted sampling strategies to update policies, whereas our work focuses on identifying key states through high-return sampling to reconstruct the memory graph.

\subsection{Local Lipschitzness for Robust Reinforcement Learning}
In RL, the data-based learning approach is notorious for not guaranteeing stability, making the control system useless and potentially dangerous. In fact, in many real-world problems, both the state and action spaces are continuous. Naturally, the Lipschitz assumptions hold. This means that when similar actions follow similar states, their effects will be similar. Thus, numerous studies have focused on the smoothness of the policy and value functions by introducing Lipschitz constraints. For policy smoothness, researchers \cite{thodoroff2018temporal, mysore2021regularizing} proposed temporal and spatial smoothness regularization terms to penalize policies when actions taken on the next state were significantly dissimilar from actions taken on the current state. Besides, related work \cite{pirotta2015policy} has studied how to automatically set a proper step size along the gradient direction for ensuring Lipschitz w.r.t. policy parameters. For the value function, one line of investigation focuses on Lipschitz continuity of the value function to obtain regret upper bounds \cite{osband2014model, ni2019learning} or upper bounds of action-value function's Lipschitz constant w.r.t. different components \cite{blonde2022lipschitzness}. However, as discussed in a previous study \cite{kobayashi2022l2c2}, overly strong Lipschitz constraints can cause the policy to lose its state-dependent adaptive behaviors. A more recent line of work focuses on the local Lipschitz continuous constraint \cite{kobayashi2022l2c2, blonde2022lipschitzness, gao2022robust, wang2024guided}. The most closely related prior work to our method is by Wang et al. \cite{wang2024guided}, which derived a model-inferred upper bound of Q-function’s Lipschitz constant w.r.t. input actions. This approach is particularly costly and often impractical in complex environments due to the necessity of large-scale data and the difficulty in approximating complex dynamics. Our work aims to estimate the upper bound w.r.t. input states, freeing itself from dependence on models.

\section{Methods}
\label{sec:method}
The proposed HG2P framework involves two critical components: 1) high-return (HR) sampling for graph construction and 2) the model-free Q-gradient penalty (MF-GP). Below, we provide details on the implementation and formulation of these two components.
\subsection{High-Return (HR) Sampling for Graph Construction}
In most complex and long-horizon tasks, the collected replay buffer might consist mostly of low-return trajectories with few high-return trajectories, since solving these tasks from scratch requires going through many iterations of trial and error, with few successes. Such challenging distributions in the replay buffer make the uniform sampling fail to exploit the rare but high-performing trajectories to their fullest. Therefore, among the existing goal-conditioned HRL frameworks, graph-based planning algorithms are susceptible to unbalanced distribution due to the dependency on uniform sampling. The recent study by Kim et al. \cite{kim2021landmark} utilized the random network distillation \cite{burda2019exploration} to explicitly sample the novel states, named novelty-based sampling. However, it primarily emphasizes the importance of rarely visited states, rather than the rare but high-return states that are crucial for success. To address this issue, we propose to re-weight the transitions in the replay buffer, assigning larger weights to high-return trajectories. The weighted sampling strategy is expected to make the replay buffer, even predominantly composed of low-return trajectories, perform like an expert dataset.

\added{\textit{Motivation}: From an instance‐based theoretical perspective, the target policy is derived by imitating demonstrations generated by several past near‐optimal policies. In the context of a hierarchical reinforcement learning architecture, sampling and then mimicking subgoals (conceptualized as abstract anchors) from high-return trajectories hold significant potential for further policy improvement. We provide a detailed, step-by-step derivation of the rationale behind subgoal sampling from high-return trajectories in \ref{motivation_sampling}, "\nameref{motivation_sampling}".} Therefore, given that a set of trajectories will be collected and stored in a replay buffer, with each trajectory $\tau^{(i)}=\{s^{(i)}_{0}, s^{(i)}_{1}, s^{(i)}_{2}, \dots\}$ being collected by different behavior policy $\pi^{(i)}$. Under the assumption of an optimal lower-level policy and a deterministic MDP with Dirac delta distributions for both the initial state and final goal, we find the target higher-level policy by imitating demonstrations generated from several past near-optimal policies:
\begin{subequations}
\begin{align}
\log \pi^{(*)}&(sg_{t}|s_{t}, g; \theta^h) \varpropto \nonumber\\ &-\sum_{(i)}\sum^{\infty}_{t=0} w^{(i)}\big\Vert\underbrace{\big[\varphi(s^{(i)}_{t+1}) - \eta\varphi(s^{(i)}_{t})\big]}_{\text{demonstrations}} - \underbrace{\pi(s_{t}, g; \theta^h)}_{\text{policy}}\big\Vert_2 + {\rm const} \\
&\text{s.t.}\qquad \max_{\mathcal{W}} \sum_{(i)} w^{(i)}\sum^{\infty}_{t=0}\gamma^t R^{(i)}_{t}
\label{main_imitating_policies:b}
\end{align}
\label{main_imitating_policies}
\end{subequations}
\replaced{w}{W}here $w^{(i)}$ is the unnormalized weight assigned to the trajectory $\tau^{(i)}$, and $\mathcal{W} := \{ w^{(0)}, w^{(1)}, w^{(2)}, \dots \}$.

\added{\textit{Methodology}:}
As described in Equation \ref{main_imitating_policies}, \replaced{a well-defined weighting distribution of trajectory demonstrations, $\mathcal{W}$, is determined by the constraint shown in Equation \ref{main_imitating_policies:b}.}{we can find a well-defined weighting distribution $\mathcal{W}$ by solving:
\begin{equation}
\max_{\mathcal{W}} \sum_{(i)} w^{(i)} \sum^{\infty}_{t=0}\gamma^t R^{(i)}_{t}
\label{traj_weighting_cond}
\end{equation}}
Meanwhile\deleted{, according to Equation \ref{traj_weighting_cond}}, we can easily convert these trajectory weights $w^{(i)}$ to transitions weights $w^{(i)}_t$:
\begin{equation}
\max_{\mathcal{W}} \sum_{(i)} w^{(i)}\sum^{\infty}_{t=0}\gamma^t R^{(i)}_{t} \Longleftrightarrow \begin{cases} &\max_{\mathcal{W}} \sum_{(i)} \sum^{\infty}_{t=0}w^{(i)}_t\gamma^t R^{(i)}_{t} \\
&w^{(i)}_t \replaced{\triangleq}{\doteq} w^{(i)}
\end{cases}
\end{equation}
It is important to note that this condition holds only when all possible policies are traversed to obtain the optimal policy. In doing so, the resulting solution is assigning all the weights to a small set of trajectories with maximum return. However, we can only access partial policies from past iterations. In that case, this straightforward solution would fail. As suggested in \cite{hong2022harnessing}, we add entropy regularization to prevent the weight distribution from becoming overly concentrated:
\begin{equation}
\max_{\mathcal{W}} \sum_{(i)} w^{(i)}\sum^{\infty}_{t=0}\gamma^t R^{(i)}_{t} - \alpha \sum_{(i)} w^{(i)}\log w^{(i)}
\label{traj_weighting}
\end{equation}
\replaced{w}{W}here $\alpha \in \mathbb{R}^+$ is a temperature parameter controlling the penalty strength.

\added{
By solving the Boltzmann distribution in Equation \ref{traj_weighting} and normalizing the weights by enforcing the constraint that they sum to 1, we obtain the following transition distribution:
\begin{equation}
w^{(i)}_t \leftarrow \frac{\exp \left(\frac{1}{\alpha}\sum^{\infty}_{t=0}\gamma^t R^{(i)}_{t} \right)}{\sum_{(i)} T_i\exp \left(\frac{1}{\alpha}\sum^{\infty}_{t=0}\gamma^t R^{(i)}_{t}\right)}
\label{main_normalized_weighting_distribution}
\end{equation}
\replaced{w}{W}here $T_i$ denotes the length of trajectory $\tau^{(i)}$. For the detailed derivation of Equation \ref{main_normalized_weighting_distribution}, we refer the reader to \ref{derivation_transition_distribution}, "\nameref{derivation_transition_distribution}". In episodic RL, $\sum^{\infty}_{t=0} \gamma^t R^{(i)}_{t}$ is equivalent to the episodic return, that is, the cumulative sum of all rewards in a trajectory, i.e., $\sum^{\infty}_{t=0} \gamma^t R^{(i)}_{t} \Leftrightarrow \sum^{T_i}_{t=0} R^{(i)}_{t}$.
}

On the other hand, we are making the strong assumption that both the initial state and the final goal follow a Dirac delta distribution. As Zhang et al. \cite{hong2022harnessing} suggested, a trick can be employed to extend both the initial state and the final goal distribution to be stochastic. We consider varying combinations of the initial state and the final goal as different tasks, where trajectories starting from lucky initial states or reaching easier goals would yield higher returns compared to others. Suppose we have trajectories $\{\tau^{(i, 0)}, \dots\, \tau^{(i, j)}, \dots\}$ starting from $s^{(i)}_{0}$ and aiming to reach goal $g^{(i)}$. Firstly, since the scale of episodic returns varies across "tasks", we can normalize these episodic returns using a max-min normalization. Then, we can estimate the expected return $\bar{A}^{(i)}$ of each initial-state-final-goal pair $ \langle s^{(i)}_{0}, g^{(i)} \rangle$, using regression: $\bar{A}^{(i)} \leftarrow \arg\min_{\bar{A}}\mathbb{E}_{j}[(\sum^{\infty}_{t=0}\gamma^t R^{(i, j)}_{t} - \bar{A})^2]$, and subtract the expected return from the episodic return to remove the stochasticity:
\begin{equation}
w^{(i)}_t \leftarrow \frac{\exp \left[\frac{1}{\alpha}\left(\sum^{\infty}_{t=0}\gamma^t R^{(i)}_{t} - \bar{A}^{(i)}\right) \right]}{\sum_{(i)}T_i\exp \left[\frac{1}{\alpha}\left(\sum^{\infty}_{t=0}\gamma^t R^{(i)}_{t} - \bar{A}^{(i)}\right) \right]}
\label{final_normalized_weighting_distribution}
\end{equation}

\added{In practice, the $ \langle s^{(i)}_{0}, g^{(i)} \rangle$ pairs lie in a high-dimensional space, potentially containing features irrelevant to the reward. To address this, we first fit a random forest regressor to model the relationship between $ \langle s^{(i)}_{0}, g^{(i)} \rangle$ and the corresponding reward. Based on the resulting feature importance scores, the top-6 most informative features are selected. The reduced representation of $ \langle s^{(i)}_{0}, g^{(i)} \rangle$ after feature selection is then used as input to a linear regression model to estimate the expected return $\bar{A}^{(i)}$. Here, we employ a gradient-boosting regressor with 50 estimators when the number of samples is fewer than 250 and switch to a simple linear regressor when more samples are available.} Finally, when sampling states to build a directed graph (see Section \ref{Landmark_Sampling}), we replace uniform sampling with weighted sampling, where the transitions in the dataset are re-weighted using $w^{(i)}_t$.

\subsection{Model-free Q-Gradient Penalty}
In a recent study, Wang et al. \cite{wang2024guided} proposed to pose the Lipschitz constraint on the Q-function gradients to stabilize the Q-learning of lower-level policy. The relevant experimental results emphasized the importance of stabilizing the lower-level behavioral policy. The underlying rationale is that enhancing the robustness of the behavioral policy can prevent catastrophic failures, such as collisions or tipping over. Yet, in that study \cite{wang2024guided}, the gradient of the Q-function was clamped by a model-inferred upper bound. That means the dynamics models approximating the true transition dynamics of the environment will be learned to provide estimates of the upper bound of reward gradients with respect to the actions. However, in complex environments that closely resemble real-world scenarios with high-dimensional observation, learning the true transition dynamics of the environment will be particularly challenging and costly. In this study, we propose a model-free Q-function gradient penalty to provide an alternative method and eliminate dependence on models. Next, we begin with a proposition that serves as the cornerstone of the formula derivation.

\begin{prop}[Q-function Lipschitzness]
\label{prop:1}
Given an MDP with the deterministic dynamics of the environment, where the deterministic policy $\pi(a_t, s_t)$ and the reward function $r(s_t, a_t)$ are differentiable over their respective input spaces. The differentiability property is satisfied by using the usual neural network-based approximators. Let us assume that the upper bounds of Frobenius norm of the policy and reward gradients w.r.t. input states are $L_{\mathcal{S}}$ and $L_{r}$, respectively, i.e., $\Vert \frac{\partial s_{t+1}}{\partial s_t}\Vert_F \leq L_{\mathcal{S}} \leq 1$ and $\lVert \frac{\partial r(s_t, a_t)}{\partial{s_t}} \rVert_F \leq  L_{r}$. Then, the gradient of the learned Q-function w.r.t. states can be upper-bounded as:
\begin{equation}
\Vert \nabla_{s_t}Q(s_t,a_t) \Vert_F \leq \frac{\sqrt{N_s}L_{r}}{1-\gamma L_{\mathcal{S}}}
\end{equation}
\replaced{w}{W}here $N_s$ is defined as the dimension of the states and $\gamma$ is the discount factor.
\end{prop}
\begin{proof}
\replaced{The complete proof is provided in the Supplementary Materials.}{See Appendix A, "Proof of Proposition 1", in the ”Supplementary Materials”.}
\end{proof}

Proposition \ref{prop:1} proposes a tight upper bound. We can derive a corollary from Proposition \ref{prop:1} corresponding to the lower-level Q function.
\begin{corollary}[Lower-level Q-function Lipschitzness]
Given that, in usual goal-conditioned hierarchical reinforcement learning, the internal reward received by the lower-level agent is time-independent quantities defined as: $r^l_t=-\lVert sg_{t+1} - \eta \varphi(s_{t+1}) \rVert_2$. By applying such a preconditioner to the reward, a relatively relaxed upper bound can be obtained:
\begin{equation}
\Vert \nabla_{s_t}Q(s_t,a_t) \Vert_F \leq \frac{\sqrt{N_s}}{1-\gamma} \cdot \sup \left \{\Vert \nabla_{s_t} \pi_{\theta^h} - \eta\nabla_{s_t} \varphi (s_t ) \Vert_F\right \}
\end{equation}
\replaced{w}{W}here $\pi_{\theta^h}$ is the higher-level policy that produces a subgoal $sg_t \sim \pi_{\theta^h}$, $\varphi$ is a mapping function that maps a state to the goal space, and $\eta$ denotes a binary function with regard to the relative/absolute subgoal scheme.
\end{corollary}
\begin{proof}
\added{See \ref{proof_corollary1}, "\nameref{proof_corollary1}".}
\end{proof}

We posit that, in this work, the mapping function $\varphi: \mathcal{S} \times \mathcal{G}$ is an element-wise function that picks out the goal-related elements from the state. Thus, in $\nabla_{s_t} \varphi (s_t )$, only the selected dimensions have gradients. In such an instance, we directly obtain: $\Vert\nabla_{s_t} \varphi (s_t )\Vert_F=\sqrt{N_{sg}}$, where $N_{sg}$ is defined as the dimensionality of the subgoals. Besides, the observation of the low-level policy involves two parts: the environmental state $s_t$ and the subgoal $sg_t$. Now, we take the derivative with respect to the two parts separately:

\begin{equation}
\begin{aligned}
\Vert \nabla_{s_t}Q(\langle s_t, sg_t \rangle,a_t) \Vert_F &\leq \frac{\sqrt{N_s}}{1-\gamma} \cdot \sup \left \{\Vert \nabla_{s_t} \pi_{\theta^h} - \eta\nabla_{s_t} \varphi (s_t ) \Vert_F\right \} \\
&\leq \frac{\sqrt{N_s}}{1-\gamma} \cdot \sup \left \{\Vert \nabla_{s_t} \pi_{\theta^h} \Vert_F\right \} + \eta \sqrt{N_{sg}}
\end{aligned}
\end{equation}
\begin{equation}
\begin{aligned}
\Vert \nabla_{sg_t}Q(\langle s_t, sg_t \rangle,a_t) \Vert_F &\leq \frac{\sqrt{N_{sg}}}{1-\gamma} \cdot \sup \left \{\Vert \nabla_{sg_t} \pi_{\theta^h} - \eta\nabla_{sg_t} \varphi (s_t ) \Vert_F\right \} \\
&\leq\frac{\sqrt{N_{sg}}}{1-\gamma} \cdot \sup \left \{\sqrt{N_{sg}} + \eta \Vert\varphi (\frac{\partial{s_t}}{\partial{sg_{t}}}) \Vert_F\right \} \\
&=\frac{\sqrt{N_{sg}}}{1-\gamma} \cdot \sup \left \{\sqrt{N_{sg}} + \eta \Vert \varphi(\frac{1}{\nabla_{s_t} \pi_{\theta^h}}) \Vert_F\right \} \\
\end{aligned}
\end{equation}
\replaced{w}{W}here $N_s$ and $N_{sg}$ are defined as the dimensionality of the environmental state $s_t$ and the subgoal $sg_t$, respectively.

As such, in effect, the remaining problem in the applications is easily solvable by estimating the gradients of higher-level policies $\pi_{\theta^h}$ w.r.t. the input variable $s_t$. In practice, we use a mini-batch independently sampled from the higher-level replay buffer $\mathcal{D}_{h}$ to estimate this gradient:
\begin{equation}
\label{upper_bound_s}
\begin{aligned}
\Vert \nabla_{s_t}Q(\langle s_t, sg_t \rangle,a_t) \Vert_F &\leq \frac{\sqrt{N_s}}{1-\gamma} \cdot \sup \left \{\Vert \nabla_{s_t} \pi_{\theta^h} \Vert_F\right \} + \eta \sqrt{N_{sg}}\\
&\simeq \frac{\sqrt{N_s}}{1-\gamma} \cdot \max_{s_t, g \sim \mathcal{D}_{h}} \Vert \nabla_{s_t} \pi_{\theta^h} \Vert_F + \eta \sqrt{N_{sg}}\\
 \end{aligned}
\end{equation}
\begin{equation}
\label{upper_bound_sg}
\begin{aligned}
\Vert \nabla_{sg_t}Q(\langle s_t, sg_t \rangle,a_t) \Vert_F &\leq \frac{\sqrt{N_{sg}}}{1-\gamma} \cdot \sup \left \{\sqrt{N_{sg}} + \eta \Vert \varphi(\frac{1}{\nabla_{s_t} \pi_{\theta^h}}) \Vert_F\right \} \\
&\simeq \frac{\sqrt{N_{sg}}}{1-\gamma} \cdot \left[\sqrt{N_{sg}} + \eta \max_{s_t, g \sim \mathcal{D}_{h}} \lVert\frac{1}{\varphi(\nabla_{s_t} \pi_{\theta^h})}\rVert_F \right] \\
 \end{aligned}
\end{equation}
Here, assuming that $\nabla_{s_{t,i}} \pi_{\theta^h} \neq 0$, $\forall s_{t,i} \in \mathbb{R}$ and $\forall i \in [1, N_s] \cap {\mathbb {N}}$, the $\frac{1}{\varphi(\nabla_{s_t} \pi_{\theta^h})}$ can be computed directly by applying an inverse operation to its Jacobian matrix. Next, for brevity, we introduce the shorthands $B^Q_{s}$, $B^Q_{sg}$ for representing the two time-independent upper bounds in Eqs. \eqref{upper_bound_s} and Eqs. \eqref{upper_bound_sg}. Finally, following prior works \cite{wang2024guided}, we derive the gradient penalty term, formally defined as:
\begin{equation}
\begin{aligned}
\mathcal{L}_{gp}(\phi^l) = \lambda_{gp} \cdot & \mathbb{E}_{s_t, sg_t\sim \mathcal{D}_{l}, \atop a_t \sim\pi(s_t, sg_t;\theta^{l})}\left\{ [{\rm ReLU}(\Vert \nabla_{s_t}Q(\langle s_t, sg_t \rangle,a_t; \phi^l) \Vert_F - B^Q_{s})]^2 \right. \\
&\qquad\qquad\left.+[{\rm ReLU}(\Vert \nabla_{sg_t}Q(\langle s_t, sg_t \rangle,a_t;\phi^l) \Vert_F - B^Q_{sg})]^2 \right \}\\
\text{s.t.}\quad s_t, g \sim \mathcal{D}_{h} &\text {,}\quad {\left\{ \begin{array}{ll} B^Q_{s} {:}{=} \frac{\sqrt{N_s}}{1-\gamma} \cdot \max_{s_t, g} \Vert \nabla_{s_t} \pi_{\theta^h} \Vert_F + \eta \sqrt{N_{sg}} \\ B^Q_{sg} {:}{=}\frac{\sqrt{N_{sg}}}{1-\gamma} \cdot \left[\sqrt{N_{sg}} + \eta \max_{s_t, g} \lVert\frac{1}{\varphi(\nabla_{s_t} \pi_{\theta^h})}\rVert_F \right]  \end{array}\right. }
\end{aligned}
\label{all_mf_gp_eq}
\end{equation}
\replaced{w}{W}here $\lambda_{gp}$ is a hyperparameter controlling the effect of the gradient penalty term. Meanwhile, we introduce Gaussian noise with a mean of 0 and a variance of $\delta_{gp}$ (generally set to 2 for mazes of size $12\times12$ and 4 for $24\times24$)  to $s_t$ and $sg_t$ to enhance the generalization performance of local Lipschitz constraints. Then, the gradient penalty term is plugged into the lower-level Q-learning loss. Because the gradient penalty enforces the Lipschitz constraint on the critic, limiting its update, we have to increase the number of critic training iterations to 5, a recommended value in WGAN-GP \cite{gulrajani2017improved}, per actor iteration. Considering the computational efficiency, we apply the gradient penalty every 5 training steps, consistent with prior work \cite{wang2024guided}.

\section{Experiments}
\label{sec:experiments}

\subsection{Environment Setup}
\label{subsec: env}

We employ a suite of challenging continuous control tasks to evaluate the effectiveness of the proposed method. As shown in Fig. \ref{fig:environments}, the benchmark tasks include several long-horizon maze navigation and robotic arm manipulation tasks, which are commonly used in the HRL literature \cite{zhang2020generating, zhang2022adjacency, kim2021landmark, wang2024guided}. \textit{Maze navigation}: (a) \textbf{Ant Push}\footnote{Here, we used a variant provided by Li et al.\cite{li2021active}, in which the action space range is reduced to between -16 and 16. As the saying goes, "Soft fire makes sweet malt."}: the ant-like robot starts at the bottom corner and navigates to the top corner. However, a movable block at (0, 4) obstructs its path. The robot needs to push the block to the right to reach the goal. The goal position is fixed at (0, 8). For Ant Push, "success" is defined as reaching within an L2 distance of 1.5 from the goal, while for the other aforementioned regular mazes, the threshold is 2.5. (b \& c) \textbf{Ant Maze (U-shape)} \& \textbf{Large Ant Maze (U-shape)}: the ant-like robot starts at the bottom left corner of a '$\sqsupset$'-shaped corridor and navigates to the target location at the top left corner. The \textbf{Large} maze has a size of $24 \times 24$, twice the size of the regular maze ($12 \times 12$). \textit{Robot arm manipulation}: (d) \textbf{FetchPush}: a box is randomly placed on the table surface, and a 7-DoF fetch arm with a two-fingered parallel gripper is manipulated to move the box to the target position. Here, the gripper of the manipulator arm is locked in a closed configuration. (e) \textbf{Pusher}: a puck-shaped object is placed on a plane and must be pushed to a target position using the end effector of a 7-DOF robotic arm.

\begin{figure*}[h]
\captionsetup[subfloat]{format=hang, justification=centering}
\centering
\subfloat[Ant Push]{\includegraphics[width=0.19\textwidth]{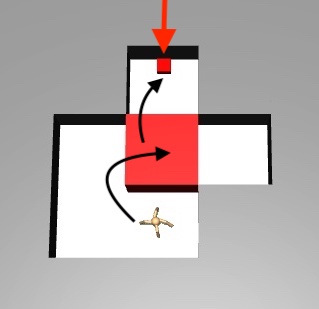}} \hspace{0.05em}
\subfloat[Ant Maze \protect\\(U-shape)]{\includegraphics[width=0.183\textwidth]{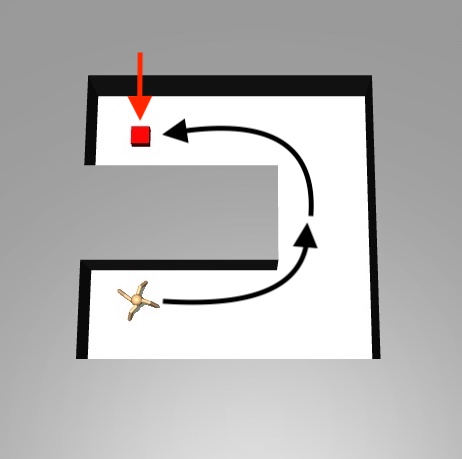}} \hspace{0.05em}
\subfloat[Large \protect\\ Ant Maze (U-shape)]{\includegraphics[width=0.1885\textwidth]{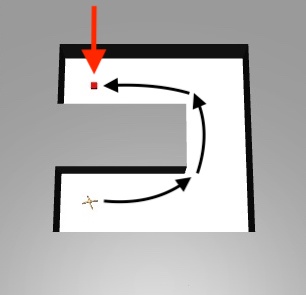}} \hspace{0.05em}
\subfloat[FetchPush]{\includegraphics[width=0.178\textwidth]{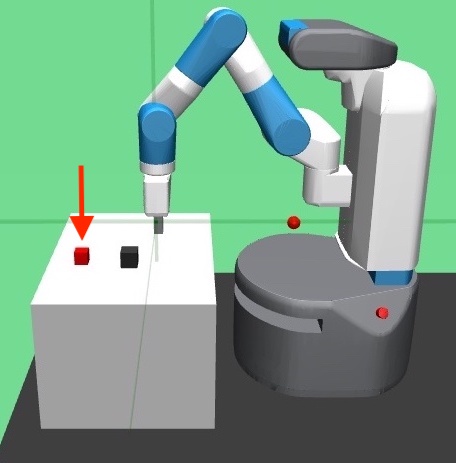}} \hspace{0.05em}
\subfloat[Pusher]{\includegraphics[width=0.1815\textwidth]{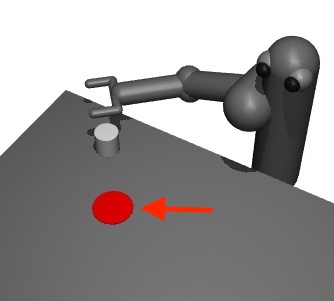}} \hspace{0.05em}
\caption{Environments used in our experiments.  The desired goal in each task is marked with a red arrow, and the black line represents a possible trajectory to reach the goal.}
\label{fig:environments}
\end{figure*}

\added{Our experiments are conducted in the aforementioned environments under \textbf{sparse} reward settings, in which the agent incurs a penalty of -1 for failing to reach the target and receives a reward of 0 upon success.} More tasks\footnote{In these tasks, including \textit{Stochastic} Ant Maze (U-shape), Ant Maze (W-shape), FetchPush, FetchPickAndPlace, etc., the proposed method demonstrates performance comparable to or slightly better than the prior state-of-the-art, ACLG+GCMR.} and comparative experiments are detailed in the "Supplementary Materials".

\subsection{Implementation and Evaluation}
\label{subsec: ie}
In all the experiments, we use the TD3 algorithm for instantiating lower- and higher-level agents, following prior works \cite{zhang2020generating, zhang2022adjacency, kim2021landmark, wang2024guided}. Also, for a fair comparison, we employ the same network structure and optimizer as that of these prior works \cite{zhang2020generating, zhang2022adjacency, kim2021landmark, wang2024guided}. Then, we integrate the proposed extensions into ACLG \cite{wang2024guided}, a state-of-the-art HRL algorithm, to validate their effectiveness. The complete pseudocode of the overall framework is provided in the "Supplementary Materials". Most hyperparameter settings remain consistent with those in vanilla ACLG, such as the learning rate, the soft update coefficient, and the standard deviation of the noise. However, a few parameters, such as the balancing coefficient $\lambda^{\rm ACLG}_{\rm LM}$ for the landmark planning term used in this study, differ from those in ACLG. In fact, parameter search reveals that the proposed high-reward sampling with a larger $\lambda^{\rm ACLG}_{\rm LM}$ yields better performance. We provide a detailed discussion of several key hyperparameters in the following sections. All the experiments are carried out on a computer with the configuration of Intel(R) Xeon(R) Gold 5220 CPU @ 2.20GHz, 8-core, 64 GB RAM. Each experiment is processed using a single GPU (Tesla V100 SXM2 32 GB). In the end, we report the mean and standard deviation of performance metrics over 5 independent runs with varying random seeds. For each run, we evaluate 10 test episodes at every $5000^{th}$ time step. More detailed experimental configurations and implementation details are available in the "Supplementary Materials".

\subsection{Hyperparameter Selection}
Firstly, we conduct experiments to investigate the effects of several key hyperparameters in ACLG \cite{wang2024guided} on the performance of the proposed framework ACLG $+$ HG2P. These hyperparameters include (1) the number of landmarks $|\mathcal{S}_{LM}^{\rm cov} \cup \mathcal{S}_{LM}^{\rm nov}|$ used for graph construction, (2) the balancing coefficient $\lambda^{\rm ACLG}_{\rm LM}$, which controls the influence of landmark-based planning, and (3) the penalty coefficient $\lambda_{gp}$, which regulates the strength of the Q-function gradient penalty. Ablation tests on hyperparameter selection are conducted in the Ant Push and Ant Maze (U-shape) tasks, with the results shown in the first and second rows of Fig.~\ref{fig:hyper_params}.

\begin{figure*}[h]
\captionsetup[subfloat]{format=hang, justification=centering}
\centering
\subfloat[Ant Push]{\includegraphics[width=0.34\textwidth]{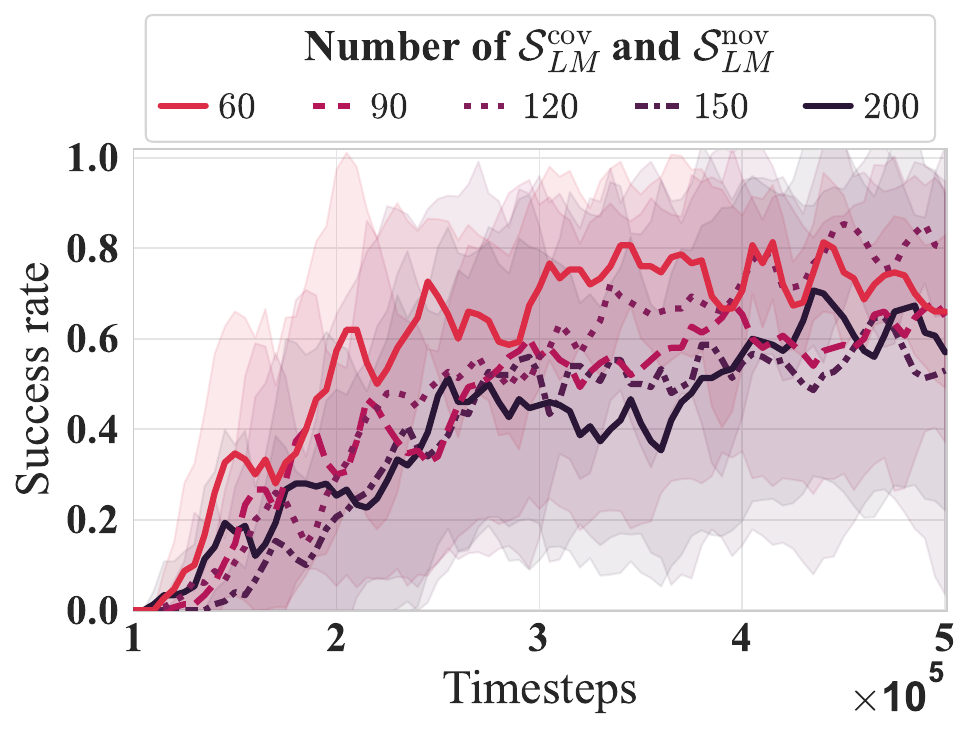}
\label{fig:antpush_ld}} \hspace*{-0.75em}
\subfloat[Ant Push]{\includegraphics[width=0.335\textwidth]{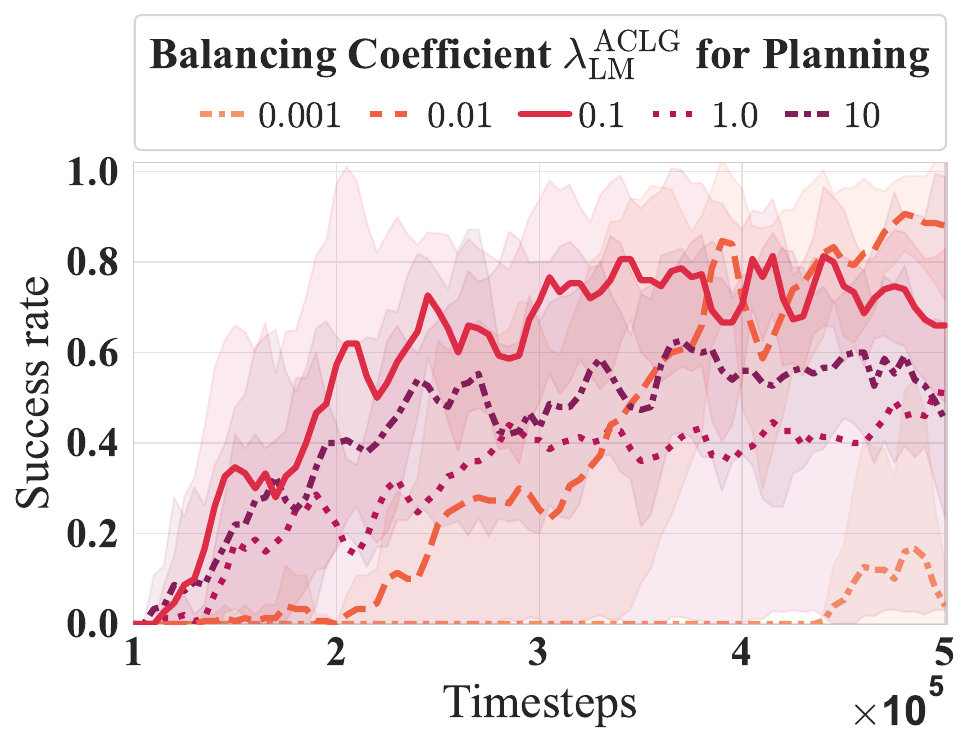}
\label{fig:antpush_llc}} \hspace*{-0.75em}
\subfloat[Ant Push]{\includegraphics[width=0.348\textwidth]{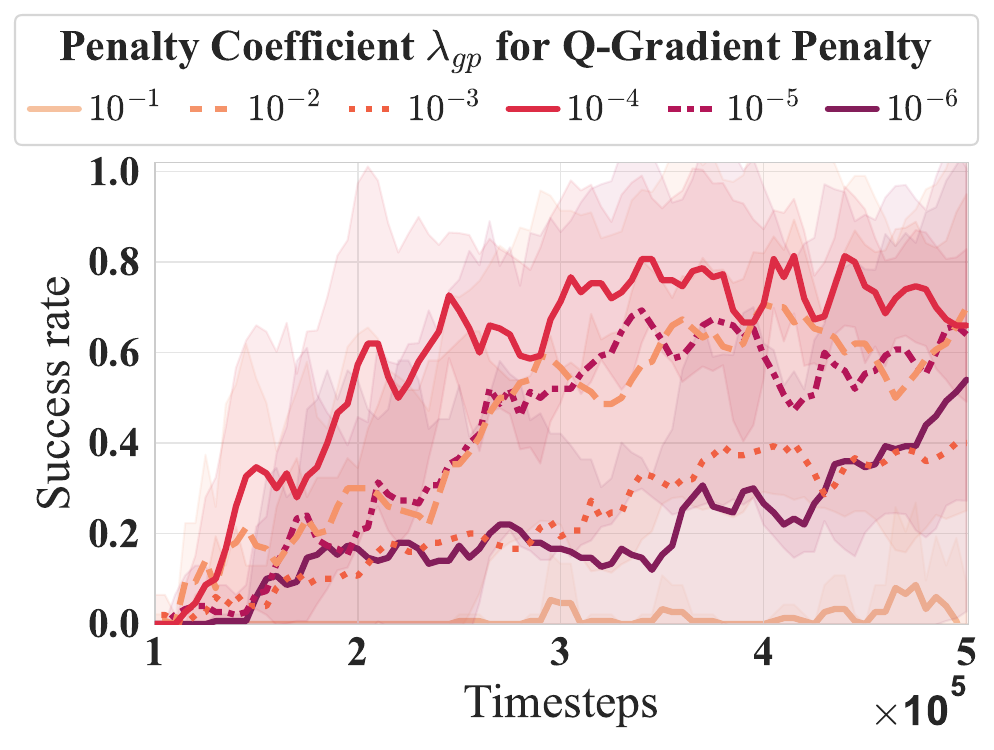}
\label{fig:antpush_gp}} \\ \vspace*{-0.8em}
\subfloat[Ant Maze (U-shape)]{\includegraphics[width=0.34\textwidth]{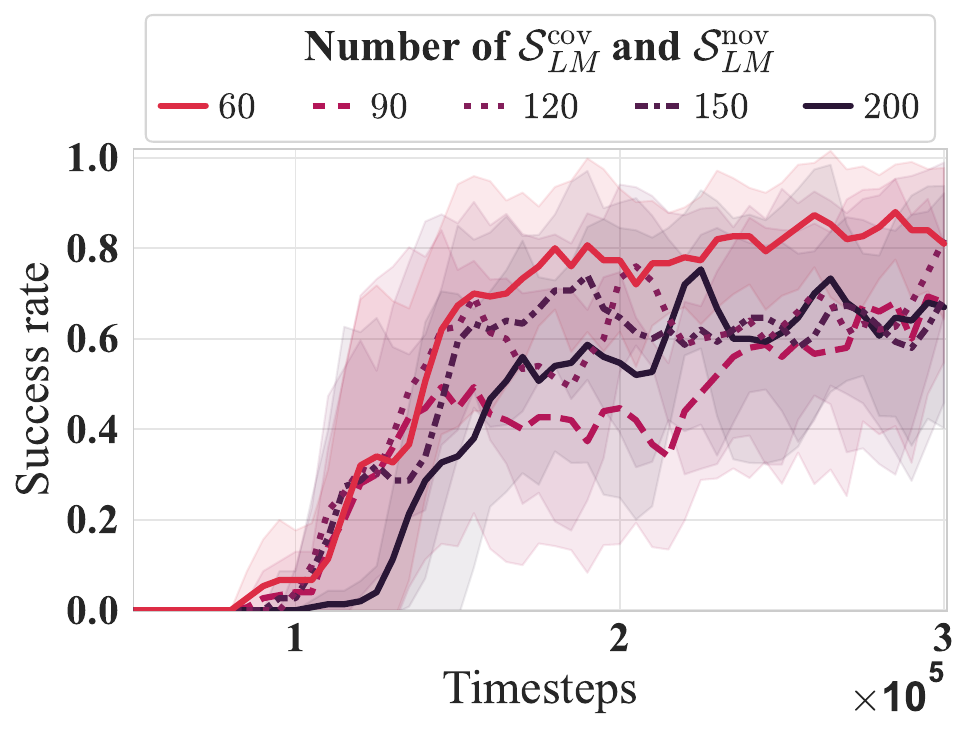}
\label{fig:antmaze_ld}} \hspace*{-0.75em}
\subfloat[Ant Maze (U-shape)]{\includegraphics[width=0.335\textwidth]{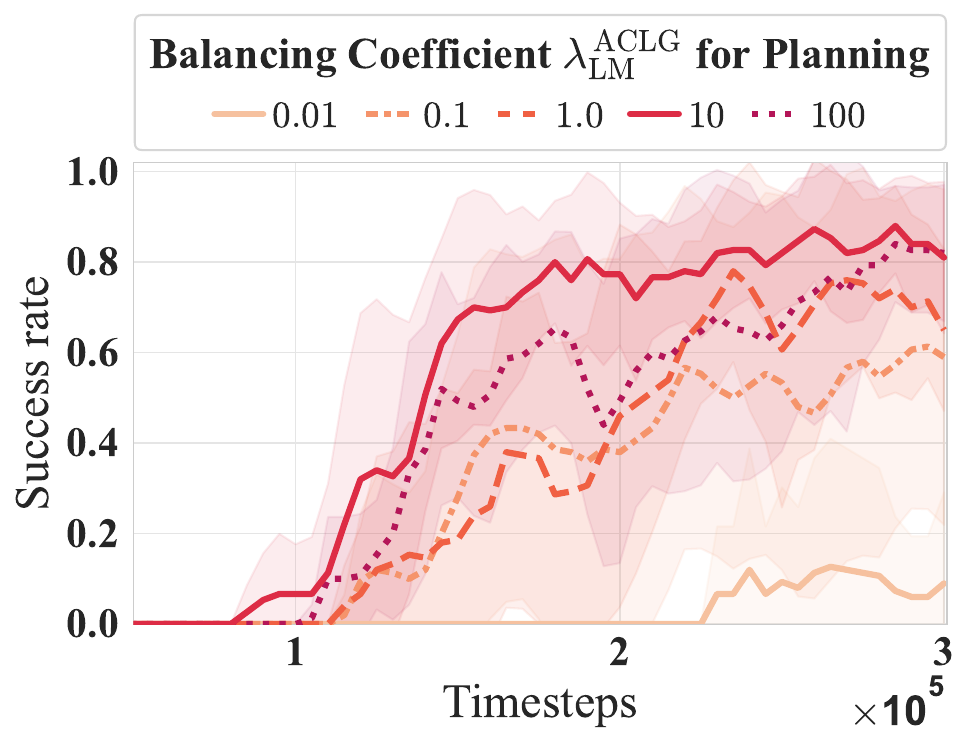}
\label{fig:antmaze_llc}} \hspace*{-0.75em}
\subfloat[Ant Maze (U-shape)]{\includegraphics[width=0.348\textwidth]{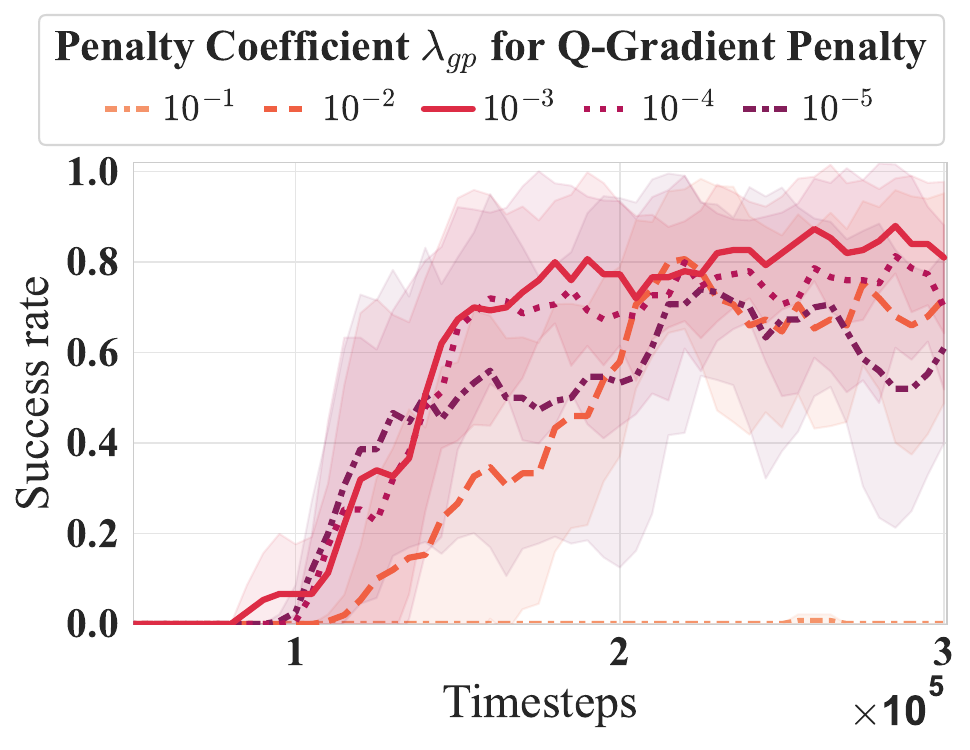}
\label{fig:antmaze_gp}}
\caption{Ablation studies on hyperarameter selection in the Ant Push (first row) and Ant Maze (U-shape) (second row) tasks. The first, second, and third columns illustrate the learning curves of the proposed framework with varying numbers of landmarks $|\mathcal{S}_{LM}^{\rm cov} \cup \mathcal{S}_{LM}^{\rm nov}|$, varying balancing coefficient $\lambda^{\rm ACLG}_{\rm LM}$, and varying penalty coefficient $\lambda_{gp}$, respectively. The success rate is averaged over five random seeds. }
\label{fig:hyper_params}
\end{figure*}

(1) \textbf{Number of landmarks} is an important and common hyper-parameter in landmark-guided HRL \cite{zhang2020generating, zhang2022adjacency, zhang2021world, huang2019mapping, kim2021landmark, lee2022dhrl, kim2023imitating, wang2024guided}. Here, we sample the same number of novelty- and coverage-based landmarks, i.e., $|\mathcal{S}_{LM}^{\rm cov}|=|\mathcal{S}_{LM}^{\rm nov}|$. As shown in Fig.~\ref{fig:hyper_params}\subref{fig:antpush_ld} and \ref{fig:hyper_params}\subref{fig:antmaze_ld}, the best learning performance is obtained when the numbers of landmarks $\mathcal{S}_{LM}^{\rm cov}$ and $\mathcal{S}_{LM}^{\rm nov}$ are 60, which is consistent with the result reported in the literature \cite{wang2024guided}.
(2) \textbf{Balancing coefficient $\lambda^{\rm ACLG}_{\rm LM}$} determines the effect of the landmark-guided planning term on the final performance. As shown in Fig.~\ref{fig:hyper_params}\subref{fig:antpush_llc} and \ref{fig:hyper_params}\subref{fig:antmaze_llc}, we find that a $\lambda^{\rm ACLG}_{\rm LM}$ value that is too small results in a catastrophic performance drop, which shows the importance of the planning term. The best performance is achieved when $\lambda^{\rm ACLG}_{\rm LM}$ is set to 0.1 for Ant Push and 10 for Ant Maze (U-shape), respectively. (3) \textbf{Penalty coefficient $\lambda_{gp}$} regulates the impact of the gradient penalty on the lower-level value function, as originally proposed by Wang et al. \cite{wang2024guided}. \deleted{However, the penalty term reported in the literature \cite{wang2024guided} concerns the upper bounds of the Frobenius norm of the Q-function gradients w.r.t. input actions, whereas in this study, we derive the Lipschitz constraint from the perspective of the Q-function gradients w.r.t. input states. }Due to the difference in the scale of the gradient penalty loss, $\lambda_{gp}$ needs to be re-tuned. As shown in Fig.~\ref{fig:hyper_params}\subref{fig:antpush_gp} and \ref{fig:hyper_params}\subref{fig:antmaze_gp}, the gradient penalty contributes to a higher success rate; however, an excessively large penalty will limit the agent's exploration efficiency. The recommended value for the parameter is 0.0001.

\begin{figure*}[h]
\captionsetup[subfloat]{format=hang, justification=centering}
\centering
\hspace*{11.2em}
\subfloat{\includegraphics[width=0.6\textwidth]{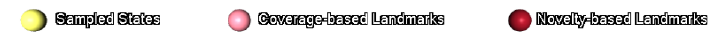}} \vspace*{-2.5em}\\
\setcounter{subfigure}{0}
\subfloat[Ant Push]{\includegraphics[width=0.34\textwidth]{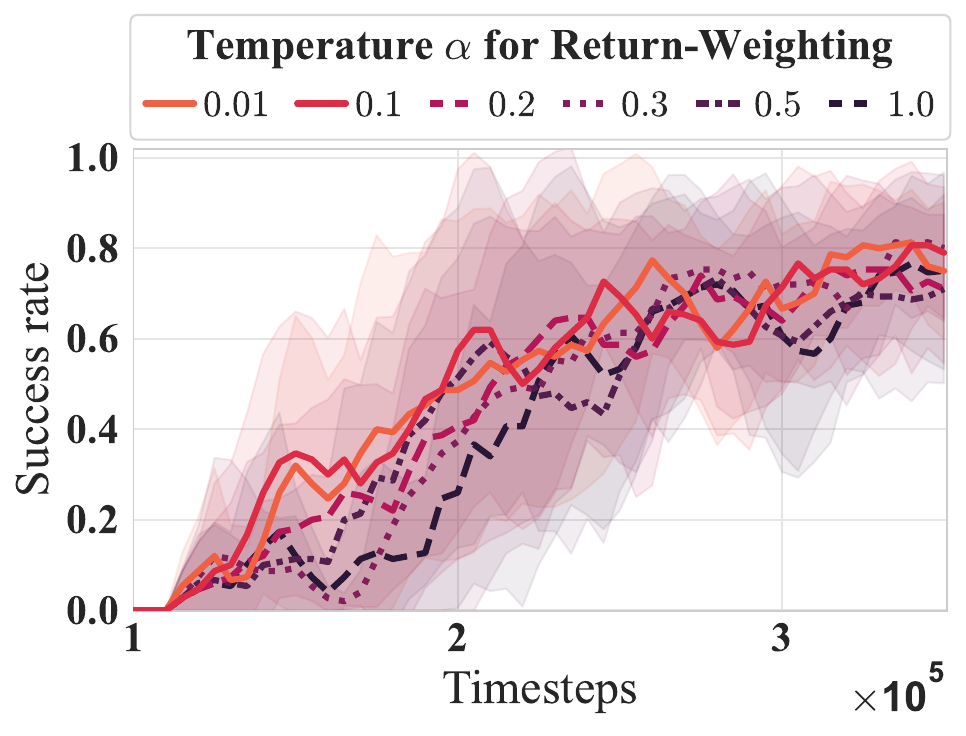}
\label{fig:antpush_wa}} \hspace*{-0.46em}
\subfloat[$\alpha=0.01$]{\includegraphics[width=0.22\textwidth]{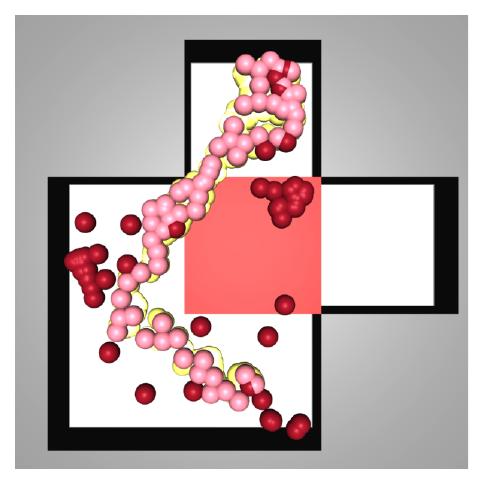}
\label{fig:}} \hspace*{-0.46em}
\subfloat[$\alpha=0.1$]{\includegraphics[width=0.22\textwidth]{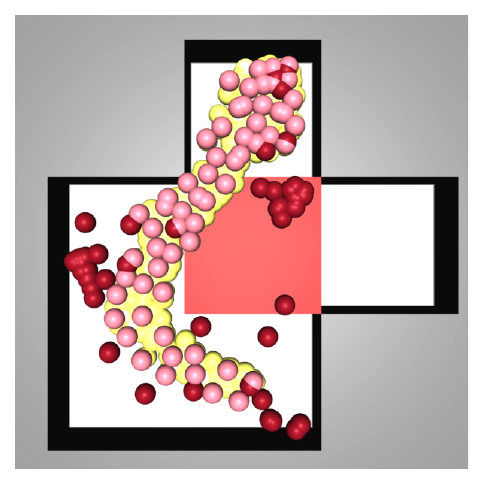}
\label{fig:}} \hspace*{-0.46em}
\subfloat[$\alpha=10$]{\includegraphics[width=0.22\textwidth]{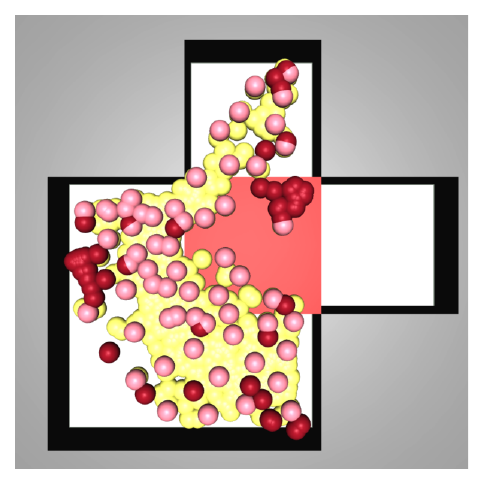}
\label{fig:}} \\ \vspace*{-0.6em}
\subfloat[Ant Maze (U-shape)]{\includegraphics[width=0.34\textwidth]{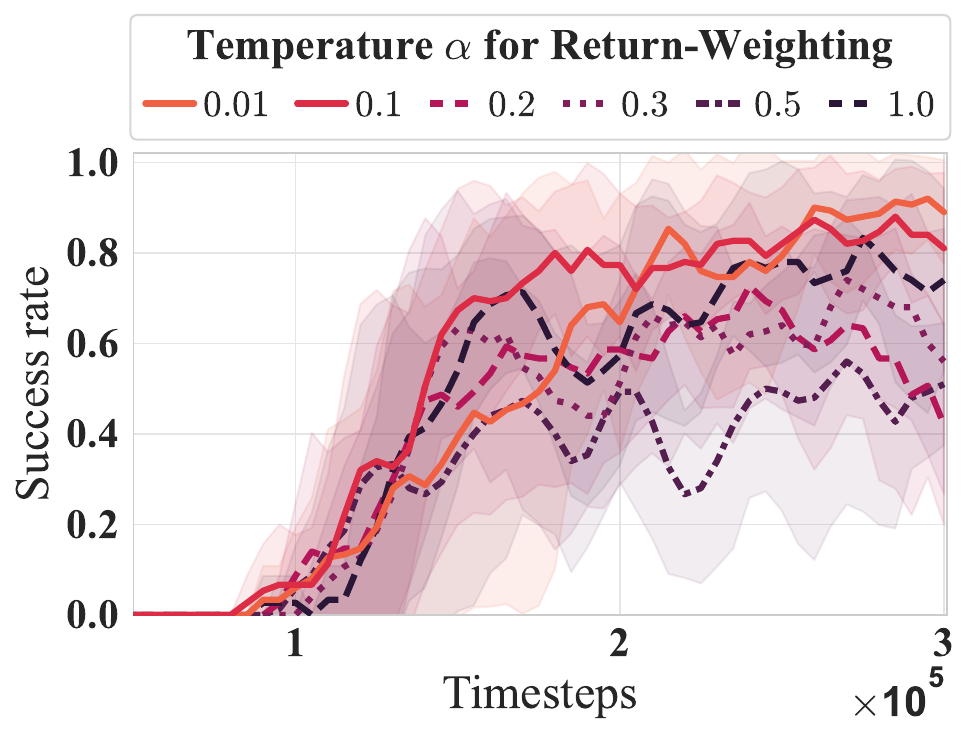}
\label{fig:antmaze_wa}} \hspace*{-0.46em}
\subfloat[$\alpha=0.01$]{\includegraphics[width=0.22\textwidth]{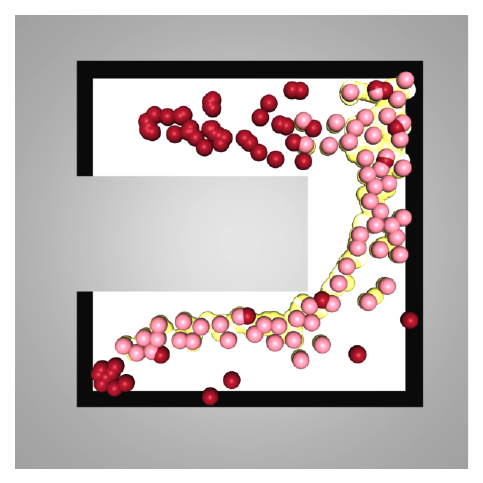}
\label{fig:}} \hspace*{-0.46em}
\subfloat[$\alpha=0.1$]{\includegraphics[width=0.22\textwidth]{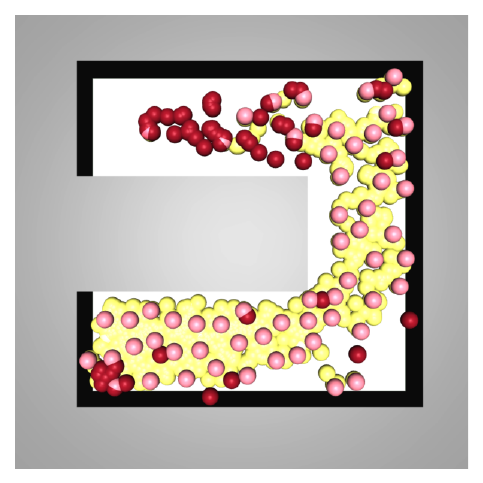}
\label{fig:}} \hspace*{-0.46em}
\subfloat[$\alpha=10$]{\includegraphics[width=0.22\textwidth]{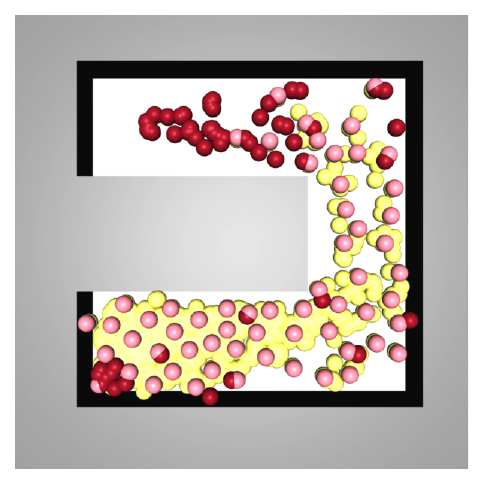}
\label{fig:}} 
\caption{The first column illustrates the performance of our method with varying temperature $\alpha$ in the \protect\subref{fig:antpush_wa} Ant Push and \protect\subref{fig:antmaze_wa} Ant Maze (U-shape) tasks. The remaining columns visualize the sampled landmarks, providing a qualitative analysis of the weighted sampling strategy at different temperatures $\alpha$. Here, we sample 1000 previously visited states (yellow balls), sparsify them to obtain 60 coverage-based landmarks (pink balls), and additionally sample 60 novelty-based landmarks (red balls).}
\label{fig:hyper_params_wa}
\end{figure*}

Then, we conduct experiments to analyze how the \textbf{high-return (HR) sampling} strategy works and how its temperature parameter $\alpha$ impacts the performance of the proposed HG2P framework. In Fig.~\ref{fig:hyper_params_wa}\subref{fig:antpush_wa} and \ref{fig:hyper_params_wa}\subref{fig:antmaze_wa}, we investigate how sensitive the choice of temperature $\alpha$ is. The results show that the proposed framework is not overly sensitive to the choice of $\alpha$, especially in near-deterministic environments with fixed goals, such as the Ant Push task\footnote{Ant Push is more suitable for episodic learning, as both the training and testing phases share the same start and goal. In contrast, in Ant Maze, the goal is randomly generated during training but fixed during testing.}. Additionally, it can be observed that smaller values of $\alpha$ help improve the asymptotic performance of the proposed framework. In fact, the temperature $\alpha$ controls the strength of entropy regularization for sampling weights. As $\alpha \rightarrow \infty$, the weighted sample distribution converges to a uniform distribution; $\alpha \rightarrow 0$, the distribution converges to a Dirac delta distribution. In Fig.~\ref{fig:hyper_params_wa}, we additionally visualize the sampled landmarks obtained through the weighted sampling strategy with values of $\alpha$ set at 0.01, 0.1, and 10. Through this qualitative analysis, it is clear that the new sampling strategy with smaller values of $\alpha$ will focus more on rare high-return trajectories. Overall, the optimal performance is achieved when $\alpha$ is set to 0.1, aligning with the results presented in the literature \cite{hong2022harnessing}.

\begin{figure*}[h]
\captionsetup[subfloat]{format=hang, justification=centering}
\centering
\subfloat[Ant Push]{\includegraphics[width=0.34\textwidth]{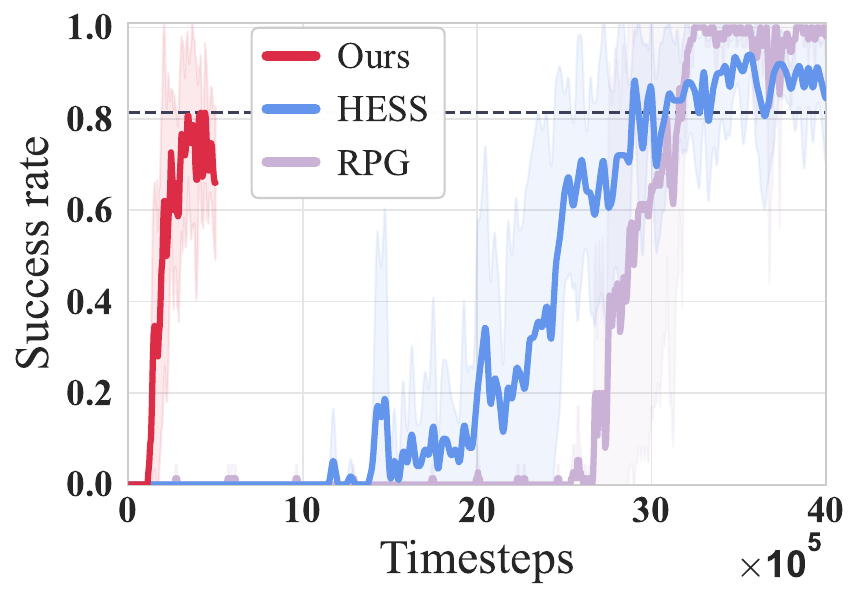}
\label{fig:antpush_hess}} \hspace*{-0.46em}
\subfloat[Ant Push]{\includegraphics[width=0.34\textwidth]{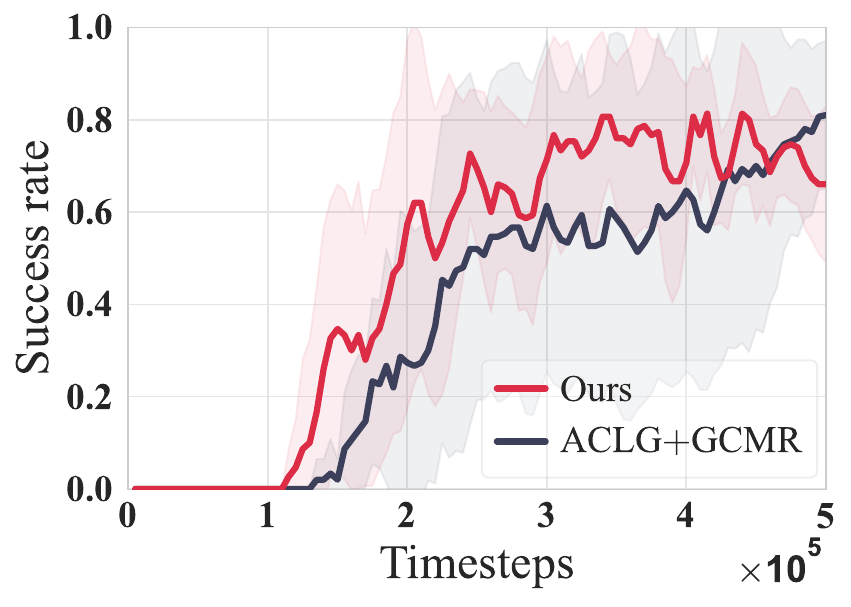}
\label{fig:antpush_gcmr}} \hspace*{-0.46em}
\subfloat[Ant Maze (U-shape)]{\includegraphics[width=0.34\textwidth]{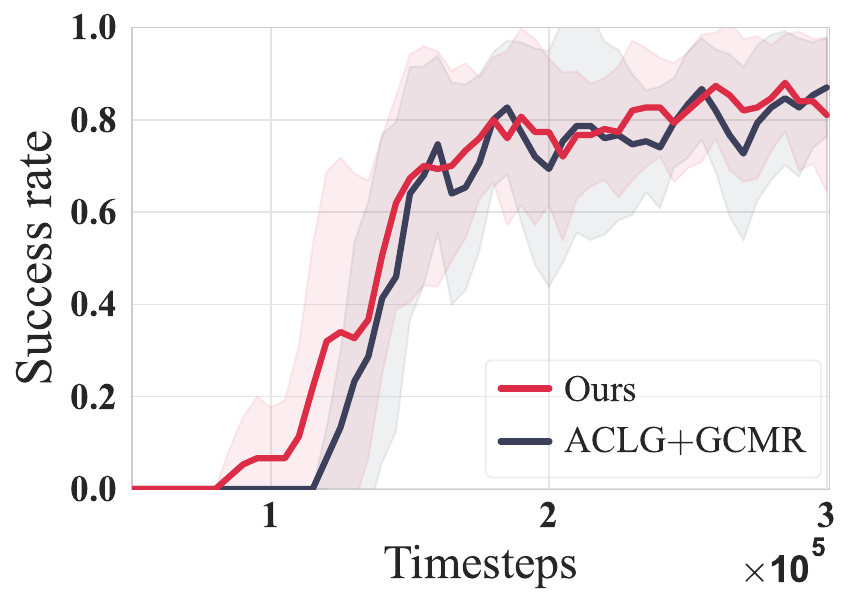} 
\label{fig:antmaze_multi_com}} \\ \vspace*{-0.6em}
\subfloat[Large Ant Maze (U-shape)]{\includegraphics[width=0.34\textwidth]{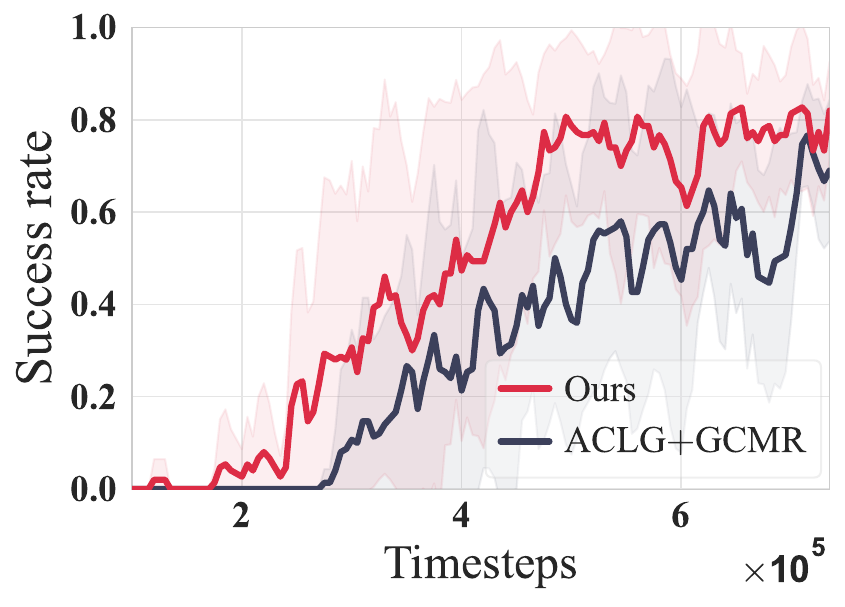} 
\label{fig:large_antmaze_multi_com}} \hspace*{-0.46em}
\subfloat[FetchPush]{\includegraphics[width=0.34\textwidth]{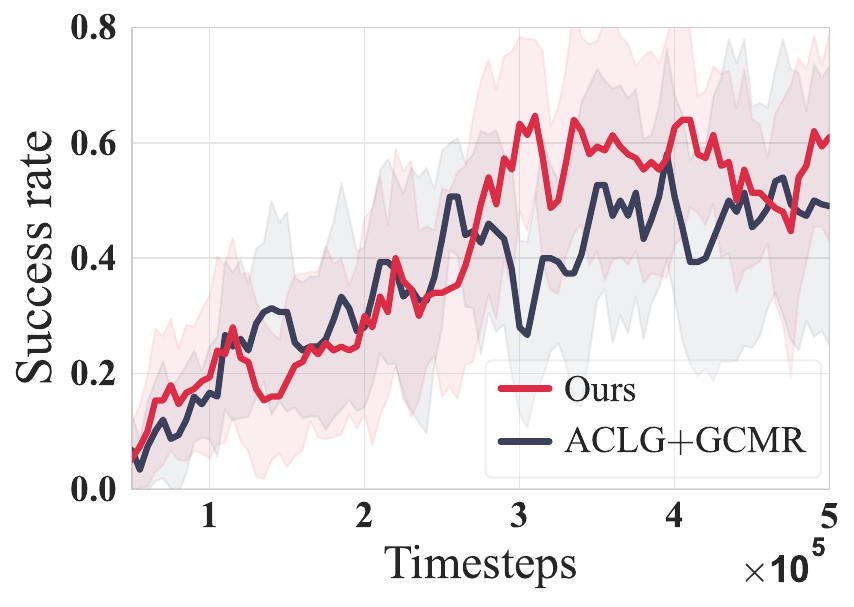}
\label{fig:pointmaze_multi_com}} \hspace*{-0.46em}
\subfloat[Pusher]{\includegraphics[width=0.34\textwidth]{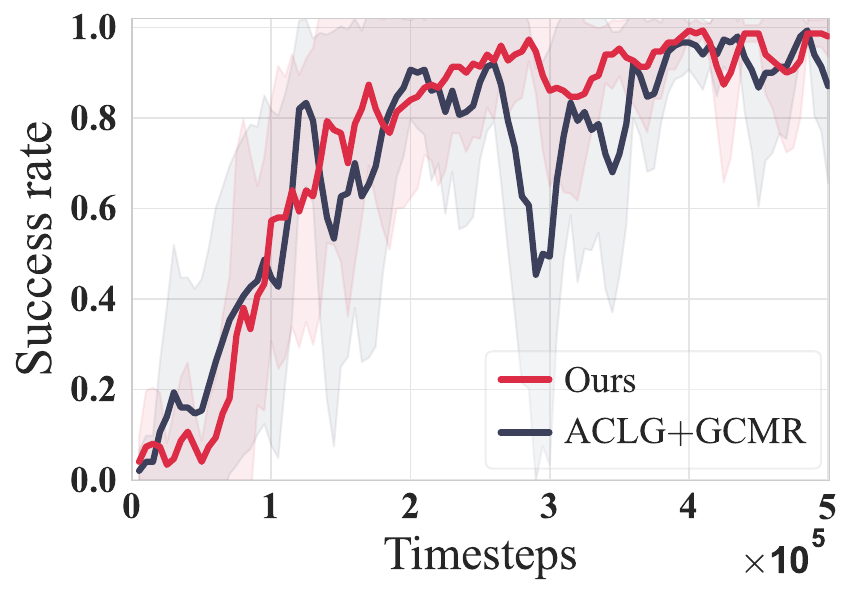} 
\label{fig:pusher_multi_com}} \hspace*{-0.46em}
\caption{Average success rate of the proposed ACLG+HG2P and the the prior state-of-the-art methods, primarily ACLG+GCMR\cite{wang2024guided}, in various continuous control tasks. The solid lines represent the mean across five runs, while the shaded regions indicate the standard deviation. Note: We reuse the experimental results pertaining to the ACLG+GCMR reported in the previous work \cite{wang2024guided}, as we employ the same random seeds (set from 0 to 5 for all tasks) and runtime environments. }
\label{fig:compare_others}
\end{figure*}
\subsection{Comparative Experiments}
To validate the effectiveness of the proposed HG2P, we integrate it with ACLG \cite{wang2024guided}, the disentangled variant of HIGL \cite{kim2021landmark}, and then compare the performance of the integrated framework ACLG+HG2P with that of the prior state-of-the-art methods. In Fig.~\ref{fig:compare_others}, we evaluate our methods against the latest state-of-the-art algorithm ACLG+GCMR\cite{wang2024guided}. Besides, we also consider two baselines: stable subgoal representation learning (HESS) \cite{li2021active} and reparameterized policy gradient (RPG)\cite{huang2023reparameterized}, both of which have been reported to achieve significant performance improvements in the Ant Push task. As shown in Fig.~\ref{fig:compare_others}\protect\subref{fig:antpush_hess} and \protect\subref{fig:antpush_gcmr}, in the Ant Push task, ACLG-based methods demonstrate significantly higher learning efficiency compared to the previous state-of-the-art methods, HESS and RPG, although the latter two exhibit better asymptotic performance. Moreover, as shown in Fig.~\ref{fig:compare_others}\protect\subref{fig:antpush_gcmr}-\protect\subref{fig:pusher_multi_com}, the proposed method also outperforms the prior state-of-the-art ACLG+GCMR in most hard-exploration tasks. We understand that such performance gain primarily stems from the more efficient Hippocampus-inspired high-reward sampling mechanism. We observe that our method is particularly effective in the earlier stages. In fact, existing methods used a first-in-first-out (FIFO) replay buffer, which ultimately leads to high-return experiences dominating the buffer in later stages.

\subsubsection{Comparison of Different Sampling Strategies: the High-Reward (HR), TopK, and Uniform.}

\begin{figure*}[h]
\captionsetup[subfloat]{format=hang, justification=centering}
\centering
\subfloat[Ant Push]{\includegraphics[width=0.4\textwidth]{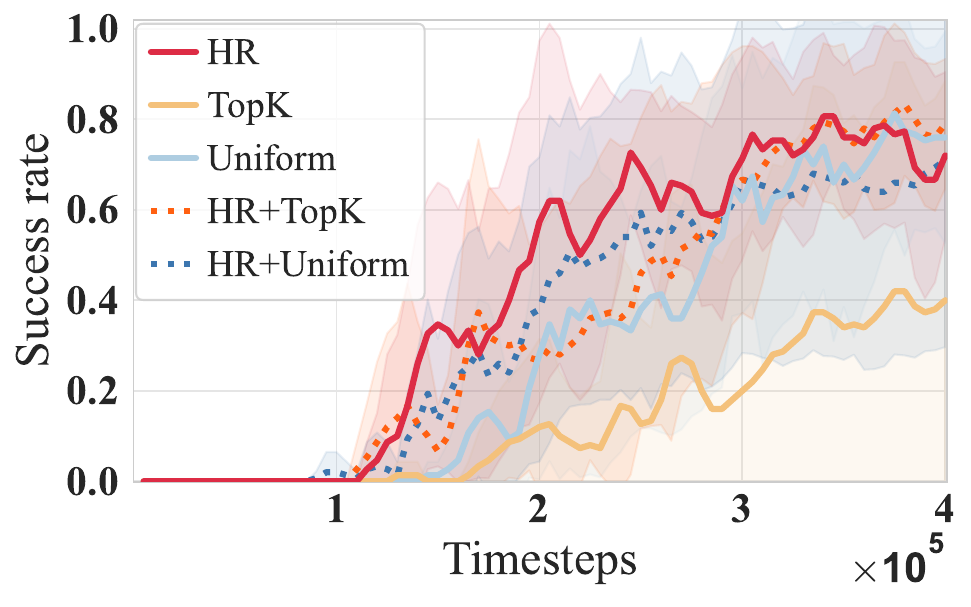}
\label{fig:antpush_smm}} \hspace*{-0.5em} 
\subfloat[HR ($\alpha=0.1$)]{\includegraphics[width=0.3\textwidth]{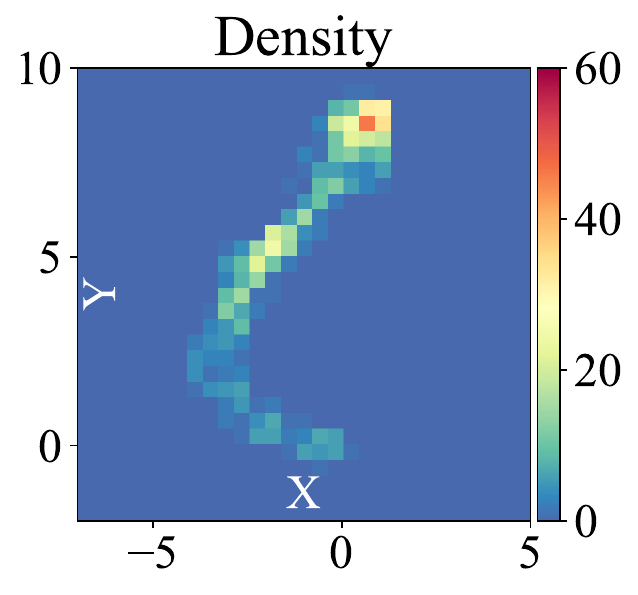}
\label{fig:antpush_smm_d1}} \hspace*{-0.5em} 
\subfloat[Uniform]{\includegraphics[width=0.3\textwidth]{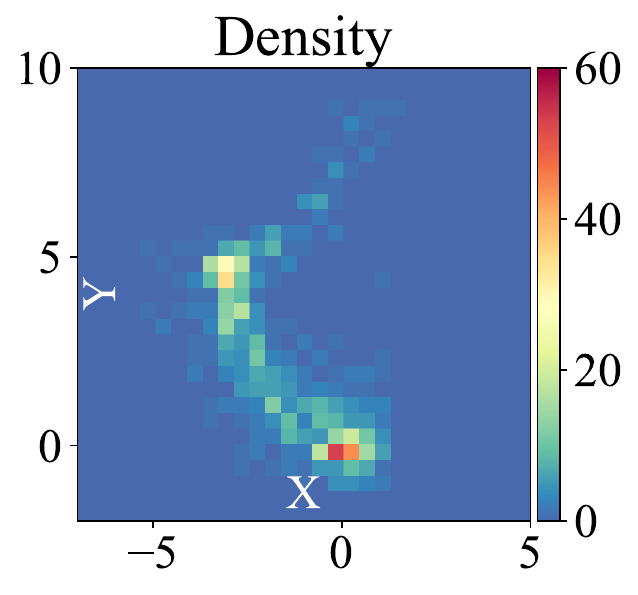}
\label{fig:antpush_smm_d2}} \hspace*{-0.8em}\\ \vspace*{-0.6em}
\subfloat[Ant Maze (U-shape)]{\includegraphics[width=0.4\textwidth]{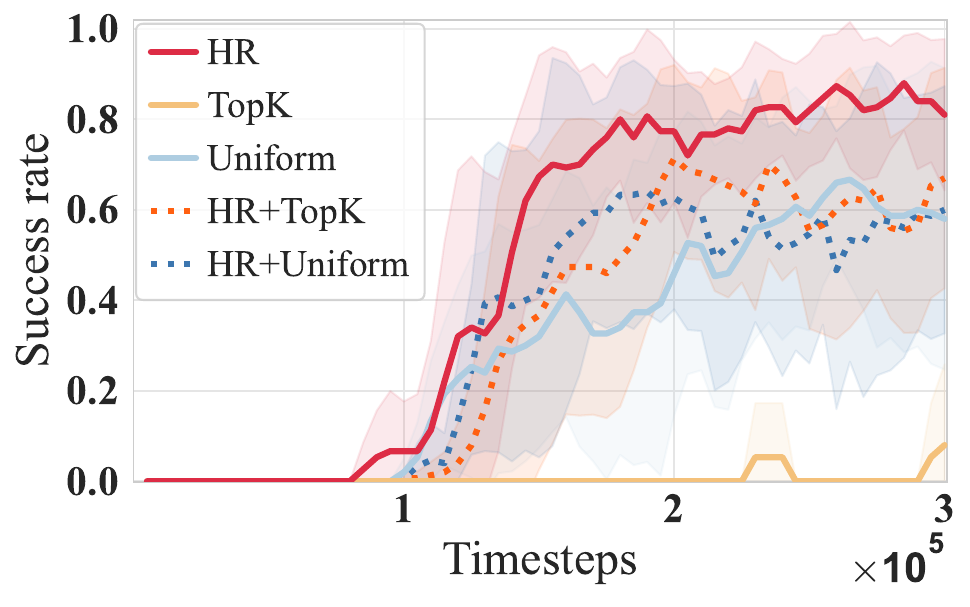}
\label{fig:antmaze_smm}} \hspace*{-0.7em} 
\subfloat[HR ($\alpha=0.1$)]{\includegraphics[width=0.3\textwidth]{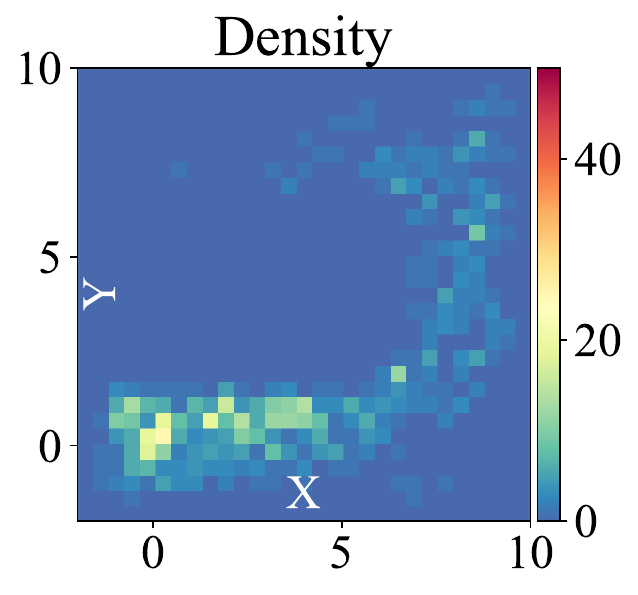}
\label{fig:antmaze_smm_d1}} \hspace*{-0.6em} 
\subfloat[Uniform]{\includegraphics[width=0.3\textwidth]{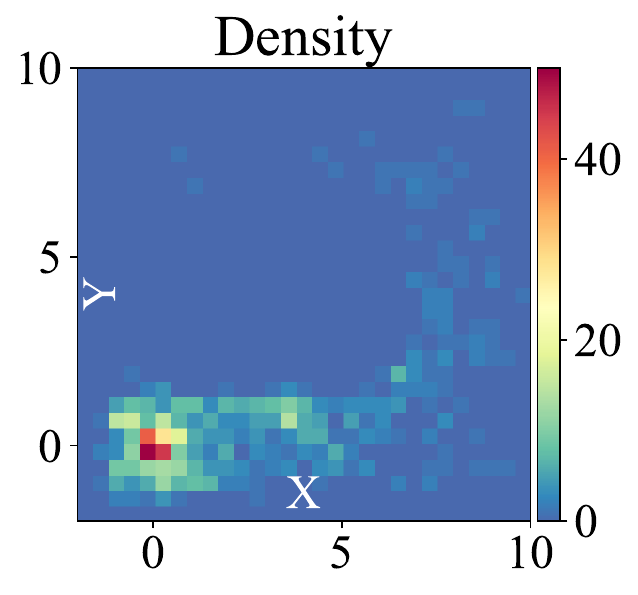}
\label{fig:antmaze_smm_d2}} \hspace*{-0.8em}
\caption{The first column illustrates the performance of our method with different sampling strategies in the \protect\subref{fig:antpush_smm} Ant Push and \protect\subref{fig:antmaze_smm} Ant Maze (U-shape) tasks. The second and third columns display the density (position-frequency) plots of sampled states using HR and uniform sampling strategies, respectively.}
\label{fig:compare_smm}
\end{figure*}

We compare the Hippocampus-inspired high-reward sampling against uniform sampling (denoted as Uniform) and percentage filtering \cite{chen2021decision} (denoted as TopK). In TopK sampling, episodes are sorted by their rewards, and the \textit{top-k} are selected for sampling. As shown in Fig.~\ref{fig:compare_smm}, our HR sampling strategy outperforms other sampling strategies, demonstrating greater improvement in both the Ant Push (Fig.~\ref{fig:compare_smm}\protect\subref{fig:antpush_smm}) and Ant Maze (U-shape) (Fig.~\ref{fig:compare_smm}\protect\subref{fig:antmaze_smm}) tasks. Interestingly, when we use mixed sampling methods, combining HR with uniform sampling (HR+Uniform) or HR with TopK (HR+TopK), the performance is still behind that of HR alone. This may be due to the introduction of low-return episodes. In Fig.~\ref{fig:compare_smm}\protect\subref{fig:antpush_smm_d1}, \protect\subref{fig:antpush_smm_d2}, \protect\subref{fig:antmaze_smm_d1}, and \protect\subref{fig:antmaze_smm_d2}, we plot the density maps (position-frequency) of visited states sampled by HR and uniform sampling, respectively. It is clear that uniform sampling tends to concentrate more around the starting point (the most frequently visited states), whereas HR sampling promotes a more even distribution across farther regions. In episodic learning scenarios, similar to how slime mold navigates a maze, HR sampling focuses on exploring high-reward paths once a high-reward point is found, rather than dispersing and branching. We understand that such differences in the sampling mechanism contribute significantly to the performance gains over other baselines.

\subsubsection{Comparison of the model-free Q-gradient penalty (MF-GP) and model-based Q-gradient penalty (MB-GP).}

\begin{figure*}[htbp]
\captionsetup[subfloat]{format=hang, justification=centering}
\centering
\subfloat[Ant Push]{\includegraphics[width=0.51\textwidth]{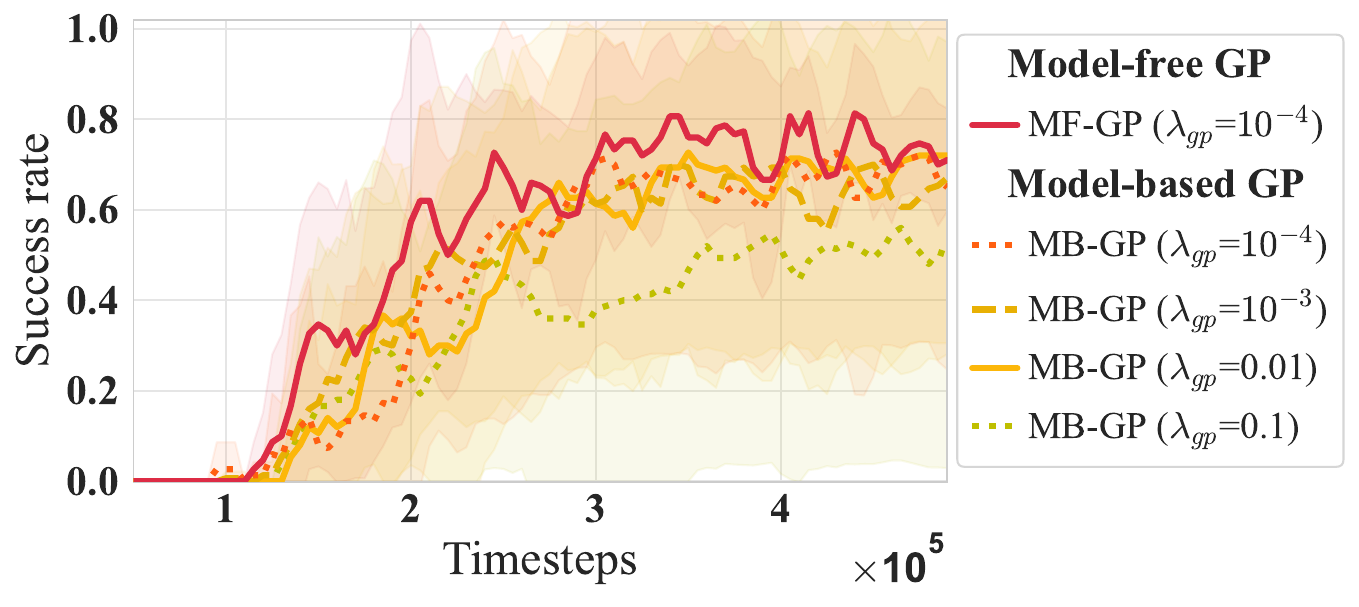}
\label{fig:antpush_wa}} \hspace*{-0.7em}
\subfloat[Ant Maze (U-shape)]{\includegraphics[width=0.51\textwidth]{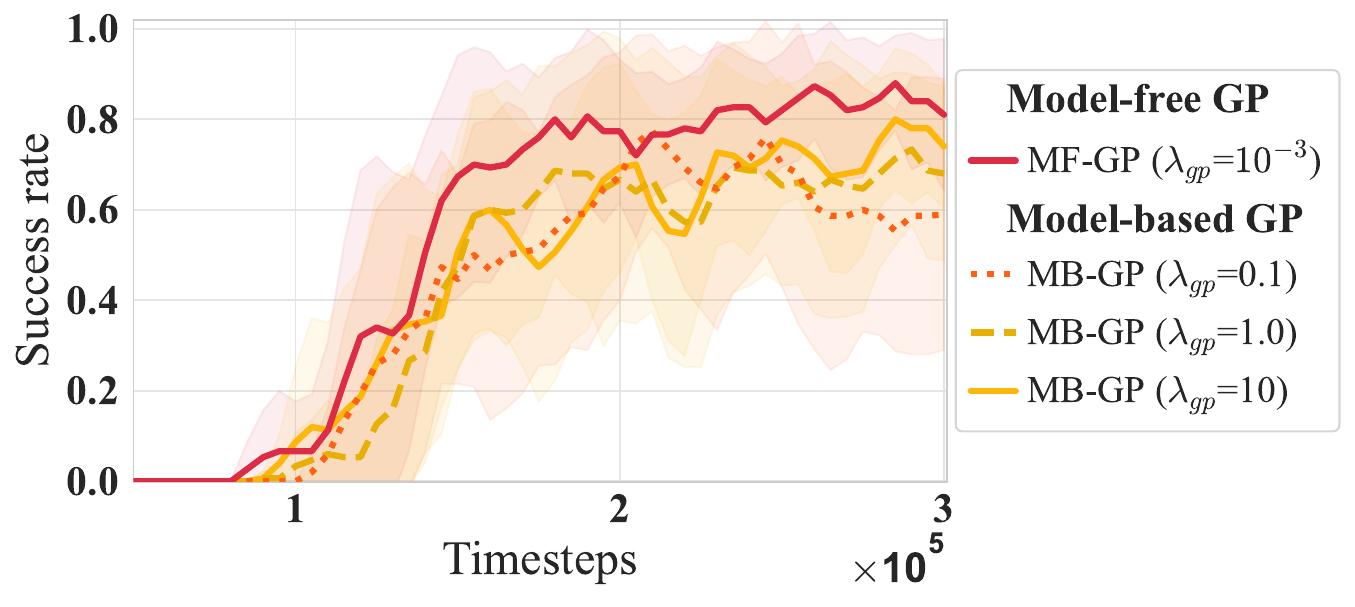}
\label{fig:}}
\caption{Comparison to the model-based Q-gradient penalty (MB-GP) proposed in prior work \cite{wang2024guided}. Due to the difference in the scale of the gradient penalty loss, we evaluate the performance of MB-GP across different values of $\lambda_{gp}$.}
\label{fig:mb_mf_gp}
\end{figure*}

Local Lipschitzness, such as the model-based Q-gradient penalty (MB-GP) proposed by Wang et al. \cite{wang2024guided}, which derives an upper bound on reward gradients using a forward dynamics model, has proven effective for robust reinforcement learning.  Our work differs in that we estimate the upper bound with respect to states and rewards, eliminating the dependence on models. In Fig.~\ref{fig:mb_mf_gp}, we compare the proposed ACLG+HG2P framework, using the new model-free gradient penalty (MF-GP), to its variant with MB-GP. As shown in Fig.~\ref{fig:mb_mf_gp}, with careful tuning, MF-GP can achieve better asymptotic performance than MB-GP. This is because deriving the upper bound of reward gradients using forward dynamics models in complex dynamical environments is challenging.
\replaced{To better demonstrate the scalability and robustness afforded by the model-free nature, in \ref{appendix:comparison_gp}, we further compare MB-GP and MF-GP in more challenging scenarios, including larger maze environments, i.e., the Large Ant Maze (U-shape) with a 24$\times$24 size, and the Stochastic Ant Maze \cite{wang2024guided}. Experimental results show that in the Large Ant Maze, our proposed method demonstrates significant advantages and robustness. In contrast, MB-GP initially benefits from local Lipschitz constraints, but its performance progressively deteriorates as exploration deepens, due to inaccuracies in the learned dynamics model. In the Stochastic Ant Maze, MF-GP underperforms compared to MB-GP, as the MB-GP employs a bootstrapped ensemble of dynamics models to approximate the stochastic transition dynamics of the environment. In comparison, MF-GP is derived under deterministic dynamics assumptions, which limits its robustness in stochastic settings. Furthermore, compared to the proposed MF-GP, MB-GP has the following advantages:}{However, MB-GP has its own advantages:} first, it is less sensitive to the hyperparameter $\lambda_{gp}$; second, it requires less computational time than the proposed MF-GP. MB-GP directly estimates the upper bound using forward dynamics models, while MF-GP involves additional computations to infer the time-independent upper bounds $B^Q_{s}$ and $B^Q_{sg}$ (see Eqs. \eqref{all_mf_gp_eq}). The time-consuming comparison between MF-GP and MB-GP can be found in Fig.~\ref{fig:compare_time}.

\subsection{Ablation Study}

\begin{figure*}[h]
\captionsetup[subfloat]{format=hang, justification=centering}
\centering
\subfloat[Ant Push]{\includegraphics[width=0.43\textwidth]{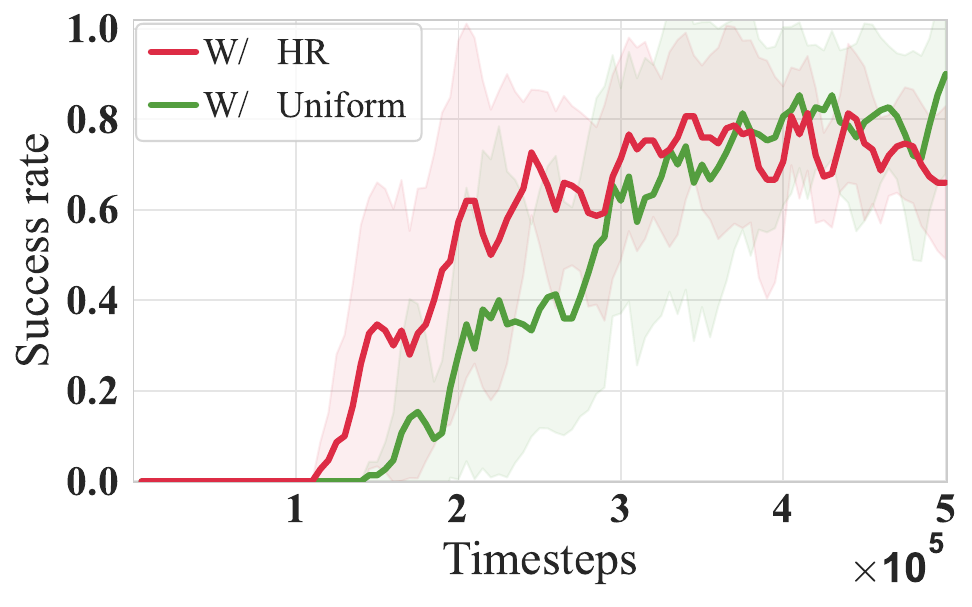}
\label{fig:antpush_abs_smm}} \hspace*{-0.7em}
\subfloat[Ant Maze (U-shape)]{\includegraphics[width=0.43\textwidth]{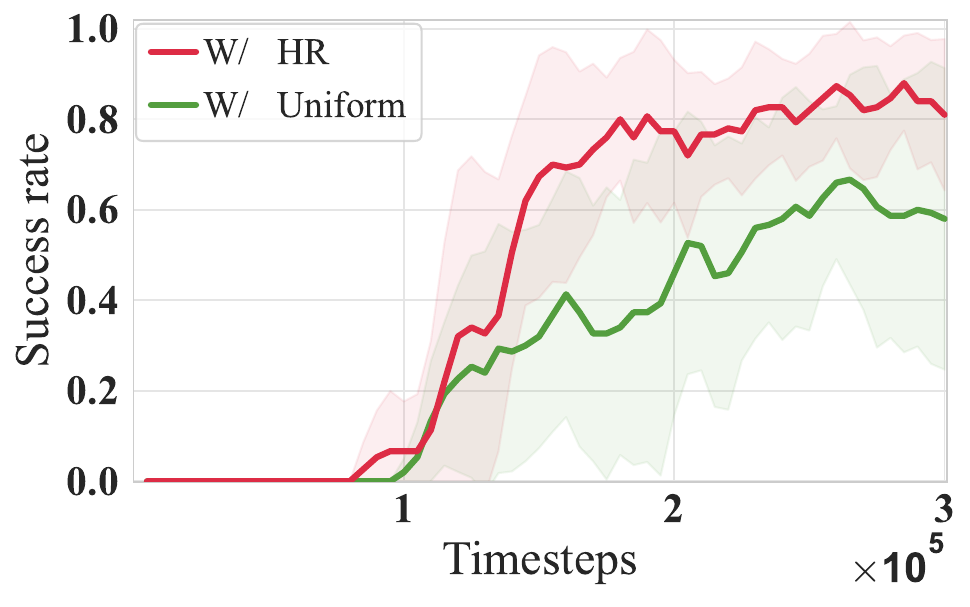}
\label{fig:antmaze_abs_smm}} \\ \vspace*{-0.6em}
\subfloat[Ant Push]{\includegraphics[width=0.43\textwidth]{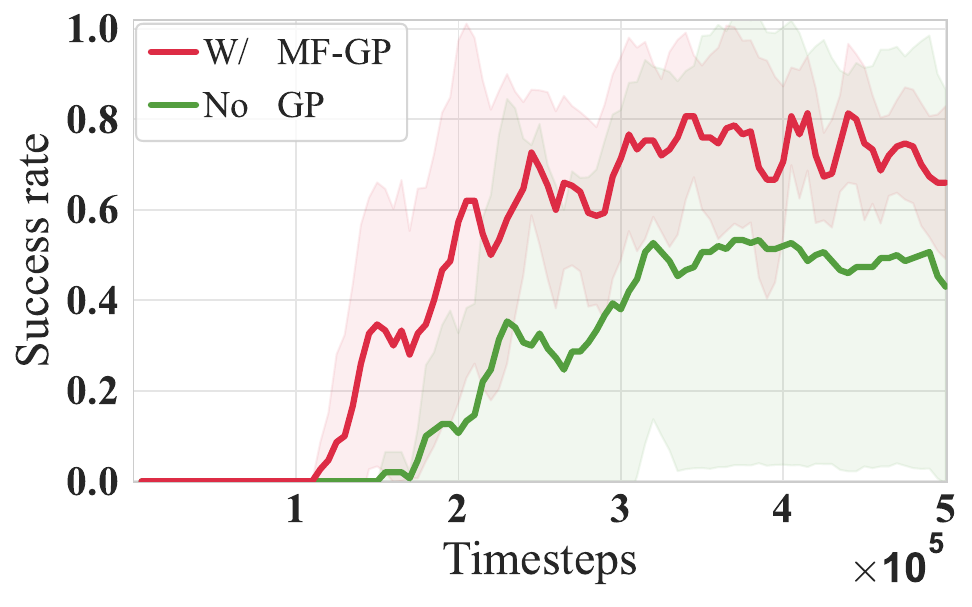}
\label{fig:antpush_abs_gp}} \hspace*{-0.7em}
\subfloat[Ant Maze (U-shape)]{\includegraphics[width=0.43\textwidth]{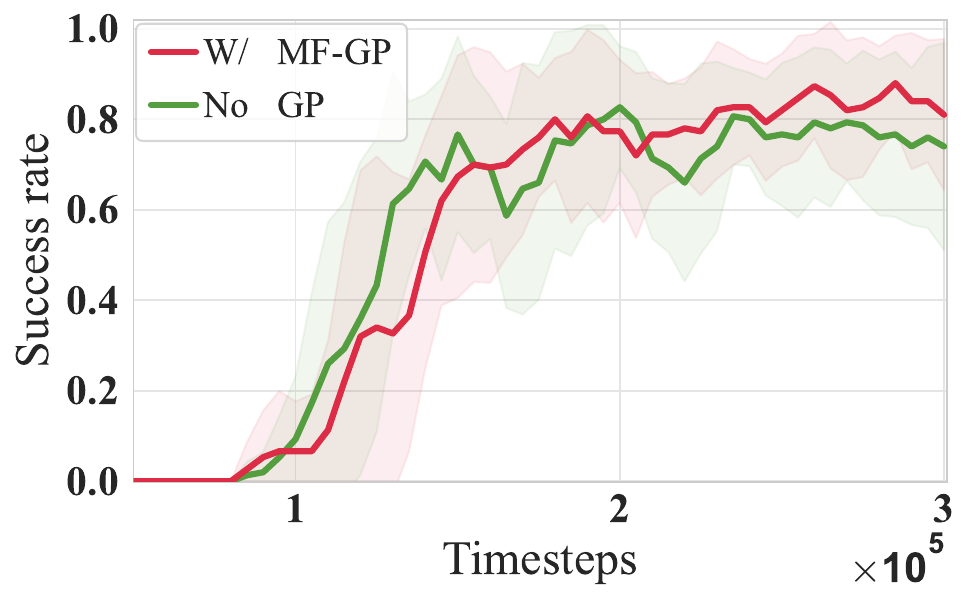}
\label{fig:antmaze_abs_gp}}
\caption{Ablation studies on two crucial components: high-reward (HR) sampling (first row) and model-free Q-gradient penalty (MF-GP) (second row). Here, we compare HR sampling with commonly used random sampling without sample weighting.}
\label{fig:abs_smm_gp}
\end{figure*}

To understand the importance and effectiveness of the newly introduced components in the proposed ACLG+HG2P framework, we conduct ablation studies on two crucial extensions: the high-reward (HR) sampling and model-free Q-gradient penalty (MF-GP), separately.

In the ablation studies of high-reward (HR) sampling, we compare the proposed ACLG+HG2P framework with HR sampling to its variant that employs uniform sampling for selected landmarks. The results are illustrated by the curves in Fig.~\ref{fig:abs_smm_gp}\protect\subref{fig:antpush_abs_smm} and \protect\subref{fig:antmaze_abs_smm}. The experimental results demonstrate a substantial advantage of high-reward sampling over uniform sampling. As discussed earlier, Fig.~\ref{fig:compare_smm} clearly illustrates how HR sampling facilitates more efficient utilization of past experiences, avoiding excessive focus on local regions in uniform sampling.
\replaced{To better highlight the superiority of HR sampling over uniform sampling, in \ref{embossed_u_maze}, we perform additional ablation experiments using a newly designed environment, the Embossed Point Maze, which is deliberately constructed with a local optimal trap. Under the sparse reward setting, HR sampling achieves a 100\% success rate, significantly outperforming uniform sampling, which yields only a 40\% success rate. This demonstrates the ability of HR sampling to effectively escape local optima by efficiently sampling and exploiting high-return trajectories. Under the dense reward setting, HR sampling performs comparably to uniform sampling, achieving performance improvement slightly earlier during training. Collectively, these results confirm}{
Therefore,} our HR approach significantly enhances the learning efficiency compared to uniform sampling.

In the ablation studies of model-free Q-gradient penalty (MF-GP), we compare with a version of the proposed framework without the MF-GP. As shown in Fig.~\ref{fig:abs_smm_gp}\protect\subref{fig:antpush_abs_gp} and \protect\subref{fig:antmaze_abs_gp}, the use of MF-GP significantly enhances learning efficiency in the challenging exploration task (i.e., Ant Push). For the Ant Maze (U-shape), both have similar asymptotic performance, although ours demonstrates a slight advantage. This may be because the benchmark is nearing or has reached saturation, where further improvements are limited. Overall, the MF-GP makes valuable contributions to robust reinforcement learning, yielding substantial performance gains.

\section{Conclusion}
\label{sec:cfw}
This study presents a novel goal-conditioned hierarchical reinforcement learning (HRL) framework called ACLG+HG2P, which incorporates the proposed HG2P into ACLG \cite{wang2024guided} to establish a new state-of-the-art HRL algorithm. The HG2P mainly consists of two crucial extensions: the Hippocampus-inspired high-reward sampling for constructing a more efficient memory graph and the model-free gradient penalty for improved robustness, eliminating the dependence on models in previous work \cite{wang2024guided}. Extensive experimental results indicate that the two extensions achieve the expected effects, significantly enhancing the performance of the HRL framework. Meanwhile, the qualitative analysis in Fig.~\ref{fig:compare_smm} elucidates why high-reward sampling is advantageous: it promotes a more even distribution across farther regions. In episodic learning scenarios, high-reward sampling emphasizes exploration around high-reward paths once high-reward points are found, rather than branching out, similar to how slime mold navigates a maze \cite{nakagaki2000maze}.

\section{\added{Limitations and Future Work}}
\label{sec:limit_cfw}
\begin{figure*}[h]
\captionsetup[subfloat]{format=hang, justification=centering}
\centering
\subfloat[\added{Overall time consumption comparison}]{\includegraphics[width=0.5\textwidth]{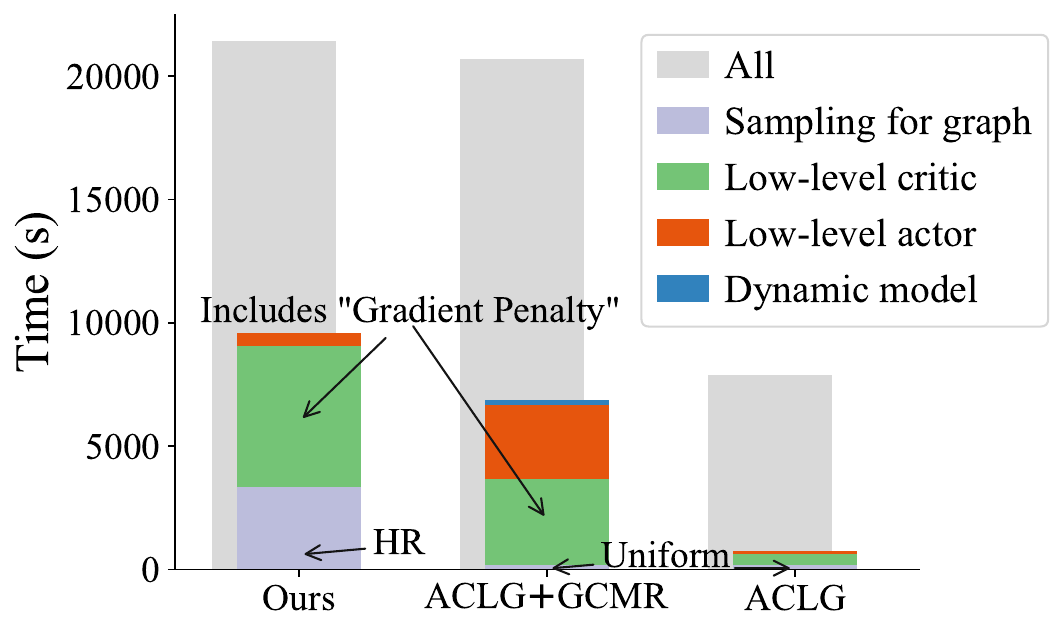}
\label{fig:overall_time_compare}} \hspace*{-0.2em}
\subfloat[\added{Time comparison of sampling methods and gradient penalties}]{\includegraphics[width=0.43\textwidth]{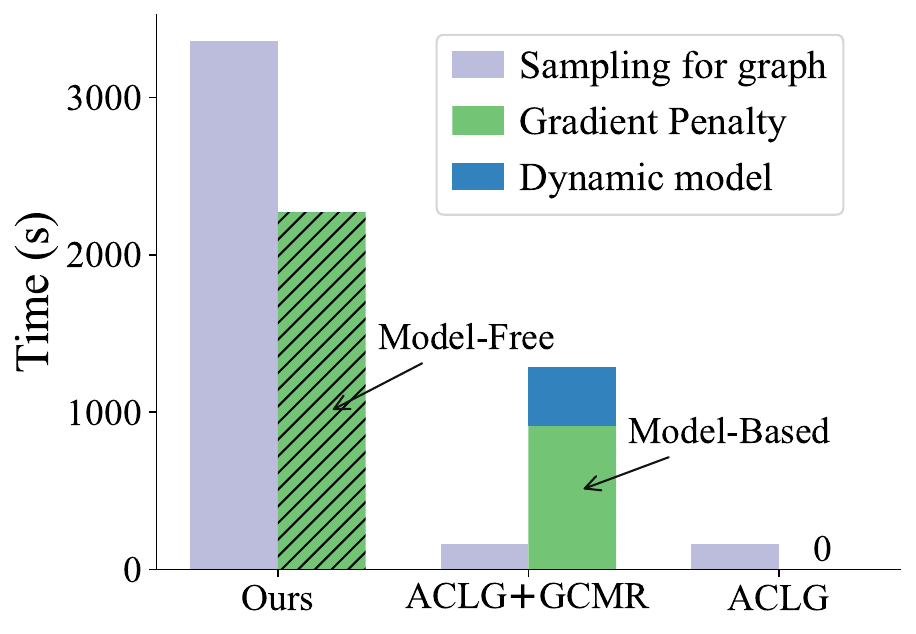}
\label{fig:detail_time_compare}}
\caption{We compare the time consumption of the proposed ACLG+HG2P, ACLG+GCMR, and ACLG, trained on the Ant Maze (U-shape) for $2 \times 10^5$ steps. The comparison includes total training time (gray), landmarks sampling time (purple), low-level critic training time (green), low-level actor training time (orange), and dynamic model training time (blue). Our method incurs extra computational cost due to HR sampling and the gradient penalty, while the GCMR primarily allocates time to model-based off-policy correction and the gradient penalty\added{, as shown in subfigure~\protect\subref{fig:overall_time_compare}}. \added{Further details are provided in subfigure~\protect\subref{fig:detail_time_compare}, which compares HR and uniform sampling methods, as well as model-based and model-free gradient penalties.}}
\label{fig:compare_time}
\end{figure*}

\replaced{Our method still has certain limitations. First,}{However,} as shown in Fig.~\ref{fig:compare_time}, such performance improvement comes at the expense of more computational cost. The quantitative analysis of computational cost shows that both high-reward sampling and the gradient penalty incur significant expenses. It is worth noting that these additional time-consuming factors occur only during the training stage and do not affect the response speed during application.\deleted{ Besides, our study still has some limitations.}
\replaced{Second, HR sampling is particularly effective in environments with sparse reward settings. However, in dense reward settings, there is a small chance of falling into local optima due to insufficient exploration. For instance, as discussed in \ref{embossed_u_maze}, the Embossed Point Maze under a dense reward setting presents a reward landscape with two peaks, one of which corresponds to a local optimum. Our experiments show that one out of five runs fails due to insufficient exploration, where the HR sampling becomes overly focused on the high-reward region and gets trapped in the local optimum. Third, as observed in Fig.~\ref{fig:compare_others}, performance degradation occurs during the later stages of training. In \ref{appendix:narrow_bottleneck}, we provide further analysis to explore the underlying cause: the narrow bottleneck of the high-reward region, resulting in too few landmarks at the corner. For practical use, we recommend dynamically adjusting the temperature parameter $\alpha$, for instance, by increasing $\alpha$ during the later stages of training to achieve a better balance between high-return and uniform sampling. Additionally, it is advisable to implement model checkpointing and an early stopping mechanism. Lastly, }{Firstly, }its applicability is primarily confined to navigational tasks, where the agent's spatial position is accurately known.\deleted{ Secondly, our method is more effective for episodic or semi-episodic learning but ineffective in fully random environments (with both random starts and random goals), potentially leading to a decrease in performance. This may be because we employed a simple trick to adapt high-reward sampling to the random distribution of both starts and goals, where we use ordinary least squares regression to estimate the expected return of each start-goal pair.} In future research, we \deleted{will explore more effective methods, such as deep neural networks, with stronger generalization capabilities to estimate the expected return of pairs. Additionally, we }will investigate the applications of the proposed framework in more general environments, beyond just navigational tasks.

\section*{Authorship contribution statement}
Haoran~Wang: Conceptualization, Methodology, Writing. Yaoru~Sun: Project administration, Supervision, Review. Zeshen~Tang: Investigation, Software, Drawing, Writing \& Editing. Haibo~Shi and Chenyuan~Jiao: Major Revision, including Additional Experiments, Data Analysis, and Statistical Analysis.

\section*{Declaration of competing interest}
The authors declare that they have no known competing financial interests or personal relationships that could have appeared to influence the work reported in this paper.

\appendix
\section{\added{Demonstrating the advantage of High-Return (HR) sampling: additional ablation experiments on the Embossed Point Maze}}
\label{embossed_u_maze}
\added{We further demonstrate the advantage of our high-return (HR) sampling method over the uniform sampling by applying it to a newly and specifically designed case. \textbf{Embossed Point Maze}: a simulated ball starts from the midpoint of the left side of the maze and navigates toward the target located at the midpoint of the right side. However, a $\sqsupset$-shaped (similar to the embossed structure of the maze) wall obstructs the direct path, causing the ball to become trapped in a local optimum.}
\begin{figure*}[htbp]
\centering
\includegraphics[width=0.8\textwidth]{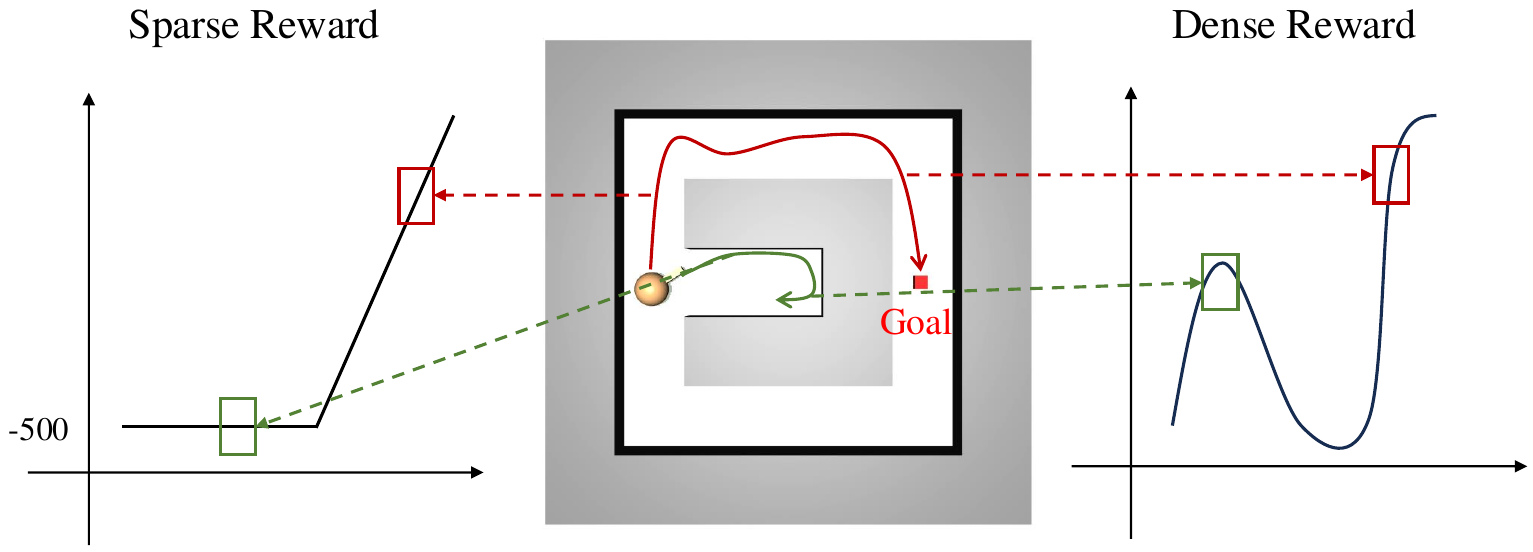}
\caption{\added{Illustration of the Embossed Point Maze, with the left and right sides corresponding to the reward landscapes under the sparse and dense reward settings, respectively.}}
\label{fig:env_embossed_u_maze}
\end{figure*}

\added{We perform ablation studies on the sampling strategy using the Embossed Point Maze. The experiments are conducted under both \textbf{sparse} and \textbf{dense} reward settings. In the \textbf{sparse} reward configuration, the agent is penalized with -1 for failing to reach the target and is granted a reward of 0 upon success. In the \textbf{dense} reward configuration, the agent incurs a penalty of the negative L2 distance to the target if it fails to reach the target and receives an auxiliary reward of 200 upon success. As a result, under the dense reward setting, the reward landscape exhibits two peaks, one of which corresponds to a local optimum, as depicted in Fig.~\ref{fig:env_embossed_u_maze}.}

\begin{figure*}[!h]
\captionsetup[subfloat]{format=hang, justification=centering}
\centering
\subfloat[Embossed Point Maze (\textbf{Sparse})]{\includegraphics[width=0.43\textwidth]{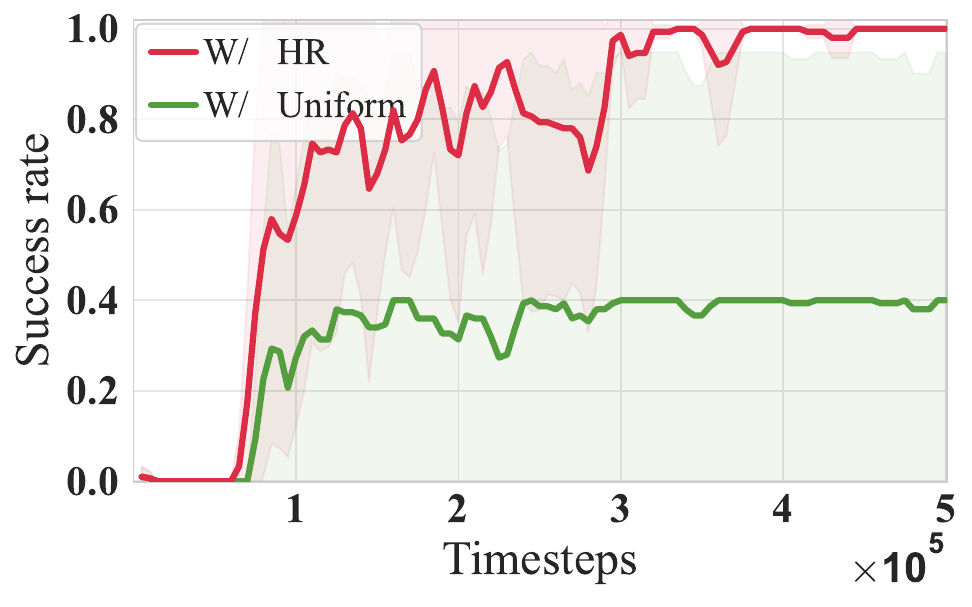}
\label{fig:env_embossed_u_maze_ablation_sparse}} \hspace*{-0.7em}
\subfloat[Embossed Point Maze (\textbf{Dense})]{\includegraphics[width=0.43\textwidth]{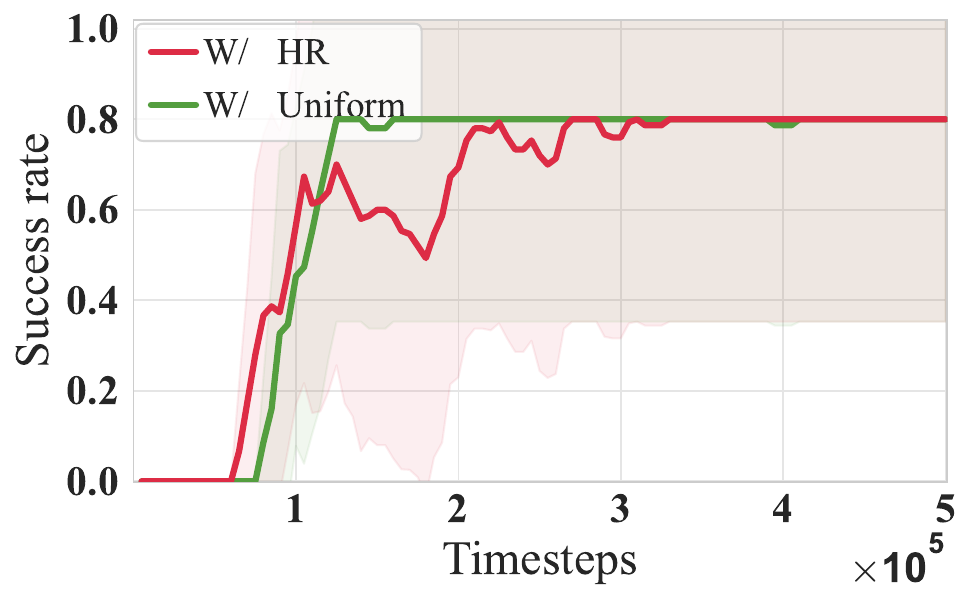}
\label{fig:env_embossed_u_maze_ablation_dense}} \\ \vspace*{-0.6em}
\caption{\added{Ablation studies on high-return (HR) sampling in the Embossed Point Maze under \protect\subref{fig:env_embossed_u_maze_ablation_sparse} sparse and \protect\subref{fig:env_embossed_u_maze_ablation_dense} dense reward settings. We compare the average success rate of the proposed framework with (W/) HR sampling and with (W/) uniform sampling without sample weighting.}}
\label{fig:env_embossed_u_maze_ablation}
\end{figure*}

\begin{figure*}[!h]
\captionsetup[subfloat]{format=hang, justification=centering}
\centering
\subfloat[\textbf{HR} sampling on Embossed Point Maze (\textbf{Sparse})]{\includegraphics[width=0.32\textwidth]{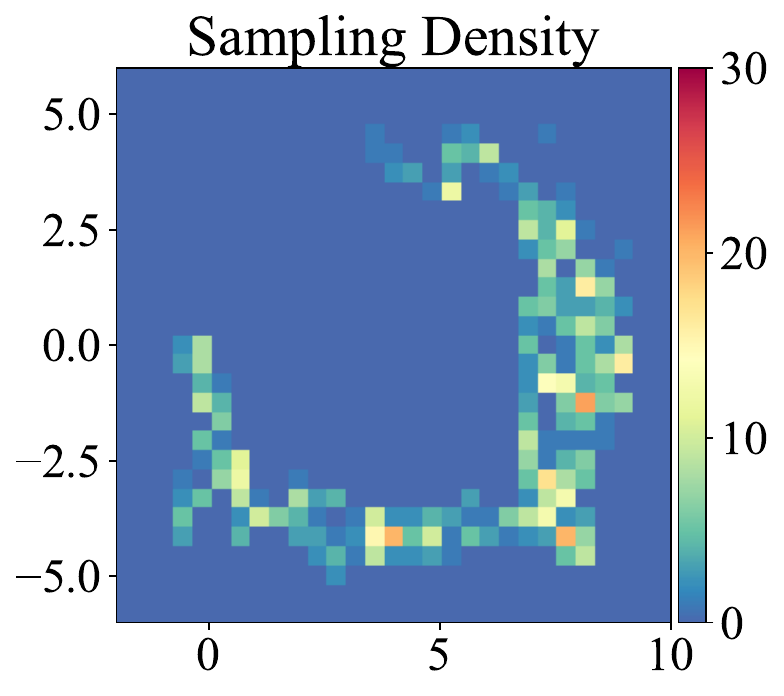}
\label{fig:env_embossed_u_maze_sampling_density_hr_sparse}} \hspace*{-0.6em}
\subfloat[\textbf{Uniform} sampling on Embossed Point Maze (\textbf{Sparse})]{\includegraphics[width=0.32\textwidth]{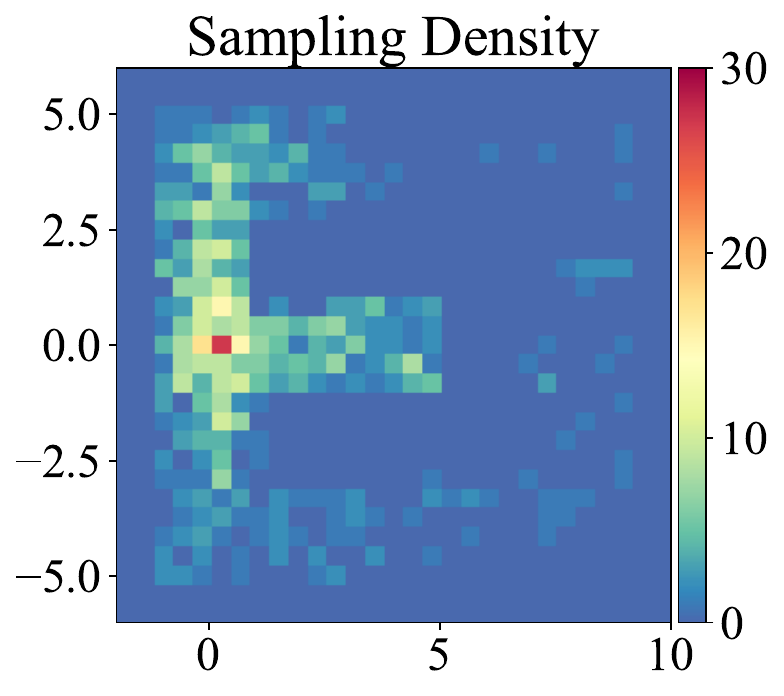}
\label{fig:env_embossed_u_maze_sampling_density_uniform_sparse}} \hspace*{-0.6em}
\subfloat[\textbf{HR} sampling on Embossed Point Maze (\textbf{Dense})]{\includegraphics[width=0.32\textwidth]{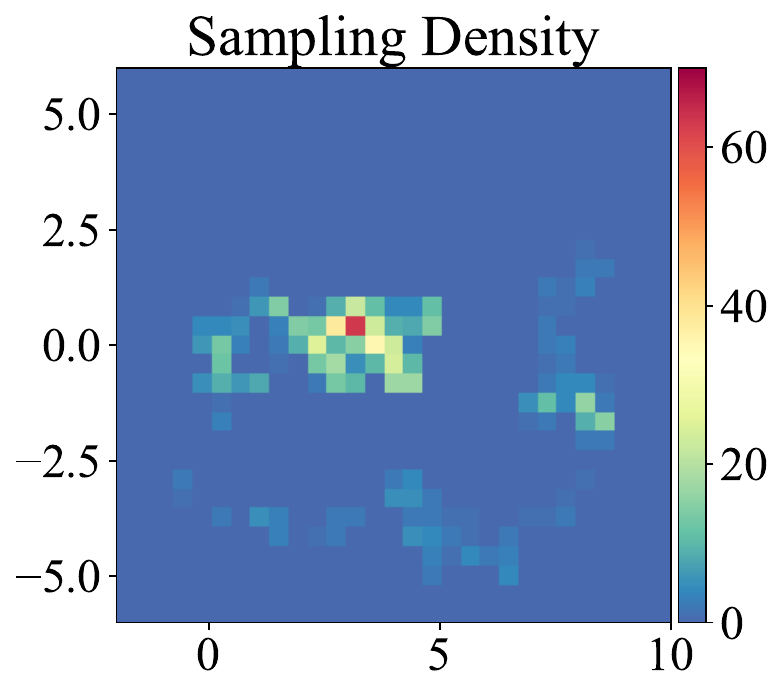}
\label{fig:env_embossed_u_maze_sampling_density_hr_dense}}
\caption{\added{Subfigures \protect\subref{fig:env_embossed_u_maze_sampling_density_hr_sparse} and \protect\subref{fig:env_embossed_u_maze_sampling_density_uniform_sparse} depict the density of sampled states at 0.08M steps under the sparse reward setting using HR and uniform sampling strategies, respectively. Subfigure \protect\subref{fig:env_embossed_u_maze_sampling_density_hr_dense} presents the state density under the dense reward setting using HR sampling. Note that the state distribution of uniform sampling is independent of the reward setting.}}
\label{fig:env_embossed_u_maze_sampling_density}
\end{figure*}
\added{We compare the proposed framework with high-return (HR) sampling to its variant that employs uniform sampling for selected landmarks, under the sparse and dense reward settings, respectively. The results are illustrated in Fig.~\ref{fig:env_embossed_u_maze_ablation}\protect\subref{fig:env_embossed_u_maze_ablation_sparse} and \protect\subref{fig:env_embossed_u_maze_ablation_dense}. It can be observed that HR sampling is more suitable for the sparse reward setting, achieving a 100\% success rate, significantly outperforming uniform sampling. However, it is worth noting that under the dense reward setting, HR sampling still tends to fall into local optima. In one out of five runs, the agent fails due to \textit{insufficient exploration} and becomes trapped in a local optimum. To more intuitively demonstrate the advantage of high-return (HR) sampling, we visualize the density of sampled states using HR and uniform sampling strategies in Fig.~\ref{fig:env_embossed_u_maze_ablation}. As shown, HR sampling focuses on exploiting high-reward trajectories once they are discovered, whereas uniform sampling exhibits a more indiscriminate and less efficient pattern. But, as illustrated in Fig.~\ref{fig:env_embossed_u_maze_sampling_density}\protect\subref{fig:env_embossed_u_maze_sampling_density_hr_dense}, when the reward landscape contains multiple peaks and local optima, insufficient exploration can cause HR sampling to overly focus on local traps.}

\section{\added{Motivation of sampling subgoal from high-return trajectories}}
\label{motivation_sampling}
\added{
We start by considering a special case where $h=1$, i.e., the higher-level policy generates a subgoal at every 1 time step interval, while the lower-level policy is optimal such that it always obtains the maximum rewards, i.e., $||\eta\varphi(s_{t+1}) - sg_{t+1}||_2=0$. Furthermore, by making the strong assumption that the MDP is deterministic and that both the initial state and the final goal follow a Dirac delta distribution, we can obtain the target policy:
\begin{equation}
\begin{aligned}
\pi^{(*)} \leftarrow \arg\max_{\pi^{(i)}} \Big[ \sum^{\infty}_{t=0} \gamma^t R^{(i)}_{t} \Big| S_0 &= \{s_0\}, \mathcal{G} = \{g\}, a_t \sim \pi^{(i)}(s_t, g), \\
& s_{t+1} \sim \mathcal{P}\left(s_{t+1}|s_t, a_{t}\right) \Big]
\end{aligned}
\end{equation}
\replaced{w}{W}here $\pi^{(i)}$ is any policy. Therefore, given that a set of trajectories will be collected and stored in a replay buffer, with each trajectory $\tau^{(i)}=\{s^{(i)}_{0}, s^{(i)}_{1}, s^{(i)}_{2}, \dots\}$ being collected by different behavior policy $\pi^{(i)}$. As a result, we find the target higher-level policy by imitating demonstrations generated from several past near-optimal policies:
\begin{equation}
\begin{aligned}
\log \pi^{(*)}&(sg_{t}|s_{t}, g; \theta^h) \varpropto \\ &-\sum_{(i)}\sum^{\infty}_{t=0} w^{(i)}\big\Vert\underbrace{\big[\varphi(s^{(i)}_{t+1}) - \eta\varphi(s^{(i)}_{t})\big]}_{\text{demonstrations}} - \underbrace{\pi(s_{t}, g; \theta^h)}_{\text{policy}}\big\Vert_2 + {\rm const} \\
&\text{s.t.}\qquad \max_{\mathcal{W}} \sum_{(i)} w^{(i)}\sum^{\infty}_{t=0}\gamma^t R^{(i)}_{t}
\label{imitating_policies}
\end{aligned}
\end{equation}
\replaced{w}{W}here $w^{(i)}$ is the unnormalized weight assigned to the trajectory $\tau^{(i)}$, and $\mathcal{W} := \{ w^{(0)}, w^{(1)}, w^{(2)}, \dots \}$. Here, the Equation \ref{imitating_policies} is derived under the assumption that the lower-level policy always completes the subgoal, i.e., $||\eta\varphi(s_{t+1}) - sg_{t+1}||_2=0$. This implies that we can require the higher-level policy to generate subgoals that follow high-return trajectories, thereby inducing the low-level policy to reproduce these high-return trajectories. This assumption can be ensured in most goal-conditioned HRL settings, where the lower-level policy is rewarded in the form of such an L2-norm distance. Besides, in the general case where $h > 1$, we can view it as mimicking highly abstract anchors of high-return trajectories.
}

\added{
Inspired by Equation \ref{imitating_policies}, it tells us imitating high-return trajectories is \textit{a sufficient condition} for obtaining the target
higher-level policy. Now, let us recall the second loss term in Equation \ref{aclg_loss}: $||sg_{t}^{\rm pseudo} - \pi(s_t, g;\theta^h) ||^2_2$, and compare it with the $\big\Vert\big[\varphi(s^{(i)}_{t+1}) - \eta\varphi(s^{(i)}_{t})\big] - \pi(s_{t}, g; \theta^h)\big\Vert_2$ in Equation \ref{imitating_policies}. Although exploring the relationship between planned graph-based waypoints and high-return trajectories is intricate, \textit{sampling $sg^{\rm plan}_t$} (the original form of the pseudo-landmark $sg_{t}^{\rm pseudo}$) \textit{from high-return trajectories holds the potential for further policy improvement}.
}

\section{\added{Derivation of the transition distribution}}
\added{
\label{derivation_transition_distribution}
For such a Boltzmann distribution in Equation \ref{traj_weighting}, the solution for the weights $\mathcal{W}$ is:
\begin{equation}
w^{(i)}_t \doteq w^{(i)} = \exp \left(\frac{1}{\alpha}\sum^{\infty}_{t=0}\gamma^t R^{(i)}_{t} -1\right)
\label{weighting_distribution}
\end{equation}
Here, to find the optimal weights, we take the derivative of function $f_{\mathcal{W}}: \sum_{(i)} w^{(i)} \sum^{\infty}_{t=0}\gamma^t R^{(i)}_{t} - \alpha \sum_{(i)} w^{(i)}\log w^{(i)}$ with respect to $w^{(i)}$ and set it to zero: $\frac{\partial f_{\mathcal{W}}}{\partial w^{(i)}} = \sum^{\infty}_{t=0} \gamma^t R^{(i)}_{t} -\alpha(1+\log w^{(i)}) = 0$. Then, we obtain the stated result shown in Equation \ref{weighting_distribution}.
Next, we further normalize the weights by enforcing the constraint that the weights sum to 1 and get a new transition distribution:
\begin{equation}
w^{(i)}_t \leftarrow \frac{\exp \left(\frac{1}{\alpha}\sum^{\infty}_{t=0}\gamma^t R^{(i)}_{t} \right)}{\sum_{(i)} T_i\exp \left(\frac{1}{\alpha}\sum^{\infty}_{t=0}\gamma^t R^{(i)}_{t}\right)} \Leftrightarrow \frac{\exp \left(\frac{1}{\alpha}\sum^{\infty}_{t=0}\gamma^t R^{(i)}_{t} -1\right)}{\sum_{(i)}T_i\exp \left(\frac{1}{\alpha}\sum^{\infty}_{t=0}\gamma^t R^{(i)}_{t} -1\right)}
\label{normalized_weighting_distribution}
\end{equation}
\replaced{w}{W}here $T_i$ denotes the length of trajectory $\tau^{(i)}$.
}
\section{\added{Proof of Corollary 1}}
\label{proof_corollary1}
\setcounter{corollary}{0}
\added{
\begin{corollary}[Lower-level Q-function Lipschitzness]
Given that, in usual goal-conditioned hierarchical reinforcement learning, the internal reward received by the lower-level agent is time-independent quantities defined as: $r^l_t=-\lVert sg_{t+1} - \eta \varphi(s_{t+1}) \rVert_2$. By applying such a preconditioner to the reward, a relatively relaxed upper bound can be obtained:
\begin{equation}
\Vert \nabla_{s_t}Q(s_t,a_t) \Vert_F \leq \frac{\sqrt{N_s}}{1-\gamma} \cdot \sup \left \{\Vert \nabla_{s_t} \pi_{\theta^h} - \eta\nabla_{s_t} \varphi (s_t ) \Vert_F\right \}
\end{equation}
\replaced{w}{W}here $\pi_{\theta^h}$ is the higher-level policy that produces a subgoal $sg_t \sim \pi_{\theta^h}$, $\varphi$ is a mapping function that maps a state to the goal space, and $\eta$ denotes a binary function with regard to the relative/absolute subgoal scheme.
\end{corollary}
}

\begin{proof}
\added{
\begin{equation}
\begin{aligned}
\Vert &\nabla_{s_t}Q(s_t,a_t) \Vert_F \\
&\leq \frac{\sqrt{N_s}}{1-\gamma L_{\mathcal{S}}} \cdot L_{r} \\
&= \frac{\sqrt{N_s}}{1-\gamma L_{\mathcal{S}}} \cdot \sup \left \{ \Vert \frac{\partial r(s_t, a_t)}{\partial{s_t}} \rVert_F\right \} \\
&= \frac{\sqrt{N_s}}{1-\gamma L_{\mathcal{S}}} \cdot \sup \left \{ \Vert \frac{-\partial\lVert sg_{t+1} - \eta \varphi(s_{t+1}) \rVert_2}{\partial{s_{t+1}}} \cdot \frac{\partial{s_{t+1}}}{\partial{s_{t}}} \rVert_F\right \} \\
&\leq \frac{\sqrt{N_s}}{1-\gamma L_{\mathcal{S}}} \cdot \sup \left \{ \Vert \frac{\partial\lVert sg_{t+1} - \eta \varphi(s_{t+1}) \rVert_2}{\partial{s_{t+1}}}\rVert_F\right \}\cdot \sup \left \{\lVert \frac{\partial{s_{t+1}}}{\partial{s_{t}}} \rVert_F\right \} \\
&= \frac{\sqrt{N_s}}{1-\gamma L_{\mathcal{S}}} \cdot \sup \left \{ \Vert \frac{ \partial\lVert sg_{t+1} - \eta \varphi(s_{t+1}) \rVert_2}{\partial{s_{t+1}}}\rVert_F\right \}\cdot L_{\mathcal{S}} \\
&= \frac{\sqrt{N_s}L_{\mathcal{S}}}{1-\gamma L_{\mathcal{S}}} \cdot \sup \left \{ \Vert \frac{(sg_{t+1} - \eta \varphi(s_{t+1}))(\nabla_{s_{t+1}} sg_{t+1} - \eta\nabla_{s_{t+1}} \varphi (s_{t+1}))}{\lVert sg_{t+1} - \eta \varphi(s_{t+1}) \rVert_2} \rVert_F\right \}\\
&\leq \frac{\sqrt{N_s}L_{\mathcal{S}}}{1-\gamma L_{\mathcal{S}}} \cdot \sup \left \{ \frac{\lVert sg_{t+1} - \eta \varphi(s_{t+1})\rVert_F}{\lVert sg_{t+1} - \eta \varphi(s_{t+1}) \rVert_2} \lVert\nabla_{s_{t+1}} sg_{t+1} - \eta\nabla_{s_{t+1}} \varphi (s_{t+1}) \rVert_F\right \}\\
\end{aligned}
\end{equation}
Because, for vectors, the 2-norm and the Frobenius norm are equivalent, we get:
\begin{equation}
\begin{aligned}
\Vert &\nabla_{s_t}Q(s_t,a_t) \Vert_F \\
&\leq \frac{\sqrt{N_s}L_{\mathcal{S}}}{1-\gamma L_{\mathcal{S}}} \cdot \sup \left \{\lVert\nabla_{s_{t+1}} sg_{t+1} - \eta\nabla_{s_{t+1}} \varphi (s_{t+1}) \rVert_F\right \}\\
\end{aligned}
\end{equation}
Next, we take the derivative of the function $\frac{L_{\mathcal{S}}}{1-\gamma L_{\mathcal{S}}}$ with respect to $L_{\mathcal{S}}$: $\nabla_{L_{\mathcal{S}}}\frac{L_{\mathcal{S}}}{1-\gamma L_{\mathcal{S}}} = \frac{1}{(1-\gamma L_{\mathcal{S}})^2} > 0$, which means that this function is an increasing function with $L_{\mathcal{S}} \leq 1$. Hence, we can obtain the claimed result:
\begin{equation}
\begin{aligned}
\Vert &\nabla_{s_t}Q(s_t,a_t) \Vert_F \\
&\leq \frac{\sqrt{N_s}}{1-\gamma} \cdot \sup \left \{\lVert\nabla_{s_{t+1}} sg_{t+1} - \eta\nabla_{s_{t+1}} \varphi (s_{t+1}) \rVert_F\right \}\\
&= \frac{\sqrt{N_s}}{1-\gamma} \cdot \sup \left \{\Vert \nabla_{s_t} \pi_{\theta^h} - \eta\nabla_{s_t} \varphi (s_t ) \Vert_F\right \}
\end{aligned}
\end{equation}
Which completes the proof.
}
\end{proof}

\section{\added{Comparison of the model-free Q-gradient penalty (MF-GP) and model-based Q-gradient penalty (MB-GP) on Large Ant Maze (U-shape) and \textit{Stochastic} Ant Maze}}
\label{appendix:comparison_gp}
\begin{figure*}[h]
\captionsetup[subfloat]{format=hang, justification=centering}
\centering
\subfloat[Large Ant Maze (U-shape)]{\includegraphics[width=0.43\textwidth]{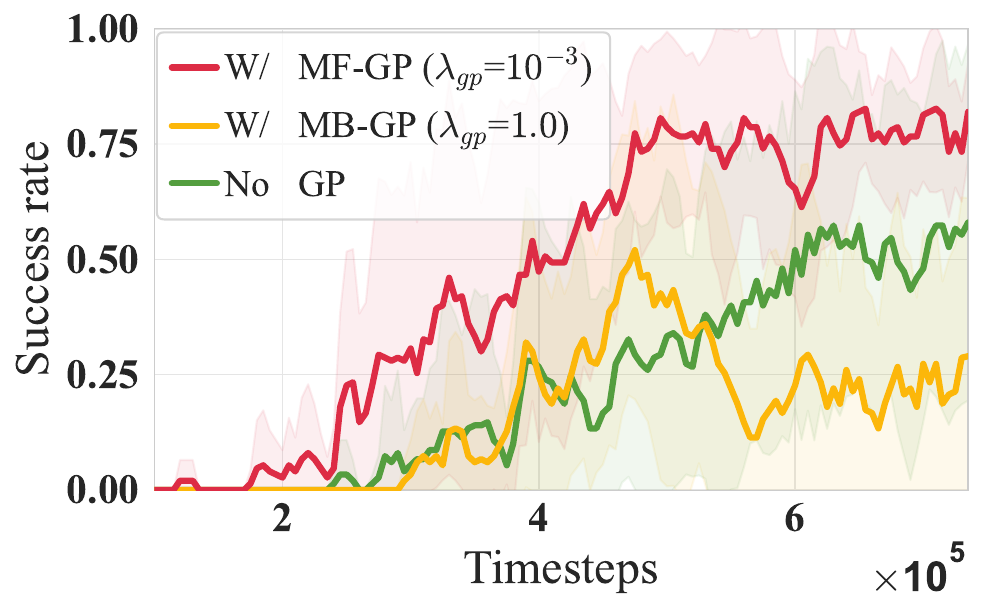}
\label{fig:appendix_comparison_gp1}} \hspace*{-0.7em}
\subfloat[\textit{Stochastic} Ant Maze]{\includegraphics[width=0.43\textwidth]{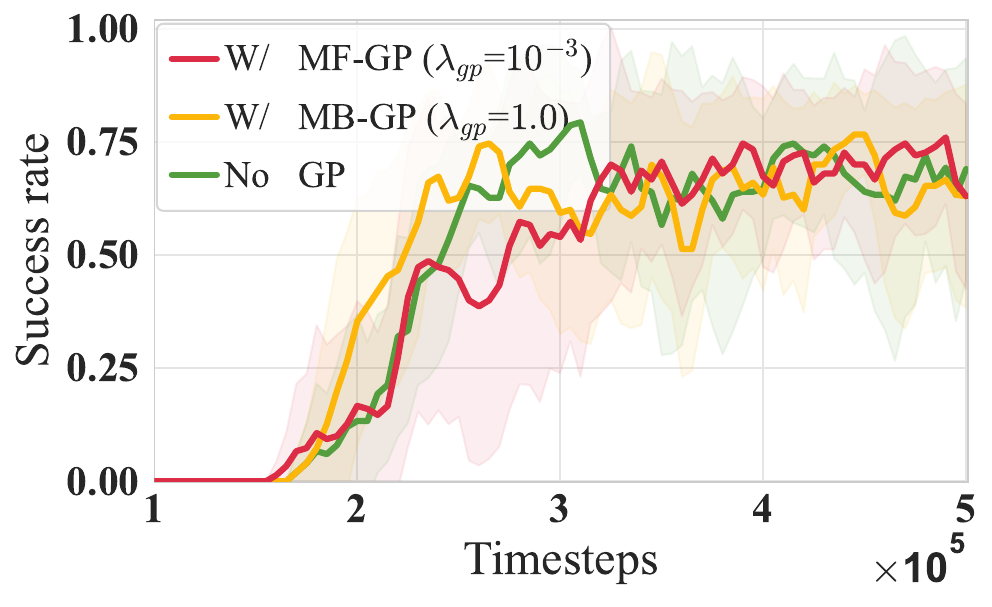}
\label{fig:appendix_comparison_gp2}}
\caption{\added{Comparison of the model-free Q-gradient penalty (MF-GP) and model-based Q-gradient penalty (MB-GP) on the Large Ant Maze (U-shape) and Stochastic Ant Maze scenarios.}}
\label{fig:appendix_comparison_gp}
\end{figure*}
\added{In prior work \cite{wang2024guided}, the authors of the model-based Q-gradient penalty (MB-GP) acknowledged that learning accurate transition dynamics models is challenging due to limited regression capability. Moreover, approximating environment dynamics to derive the upper bound is particularly costly and often impractical in complex dynamical environments. To address these limitations, we propose a model-free Q-gradient penalty (MF-GP) that eliminates the dependence on learned models. To better demonstrate the scalability and robustness of the proposed method, we further compare MB-GP and MF-GP in more extensive and challenging scenarios, including larger maze environments, such as the Large Ant Maze (U-shape) with a 24$\times$24 size, and stochastic settings, such as the Stochastic Ant Maze \cite{wang2024guided} with a 12$\times$12 size.}

\added{As shown in Fig.~\ref{fig:appendix_comparison_gp}, in the Large Ant Maze, our proposed method demonstrates significant advantages and robustness, due to its independence from learned dynamics models. In contrast, MB-GP initially benefits from local Lipschitz constraints; however, its performance progressively deteriorates as exploration deepens, owing to inaccuracies in the learned dynamics. In the Stochastic Ant Maze, MB-GP exhibits better robustness by employing a bootstrapped ensemble of dynamics models to approximate the stochastic transition dynamics of the environment. In comparison, our proposed MF-GP, which is derived under deterministic dynamics assumptions, shows reduced robustness in such stochastic settings. This provides valuable insight for our future work on enhancing the robustness of MF-GP in stochastic environments.}

\section{\added{Limitation of high-return (HR) sampling: the narrow bottleneck problem}}
\label{appendix:narrow_bottleneck}
\begin{figure*}[!h]
\captionsetup[subfloat]{format=hang, justification=centering}
\centering
\subfloat{\includegraphics[width=0.38\textwidth]{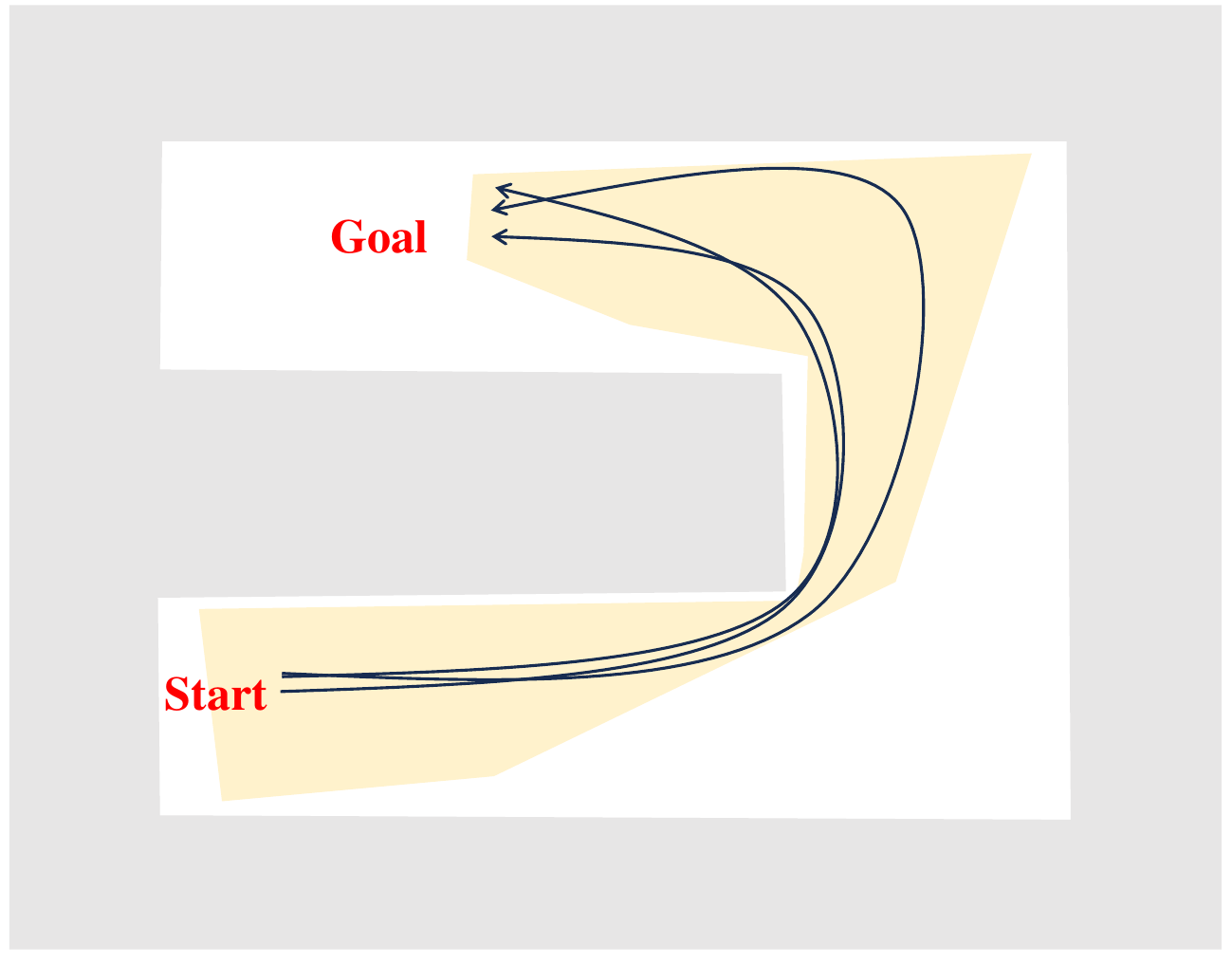}} \hspace*{0.2em}
\subfloat{\includegraphics[width=0.38\textwidth]{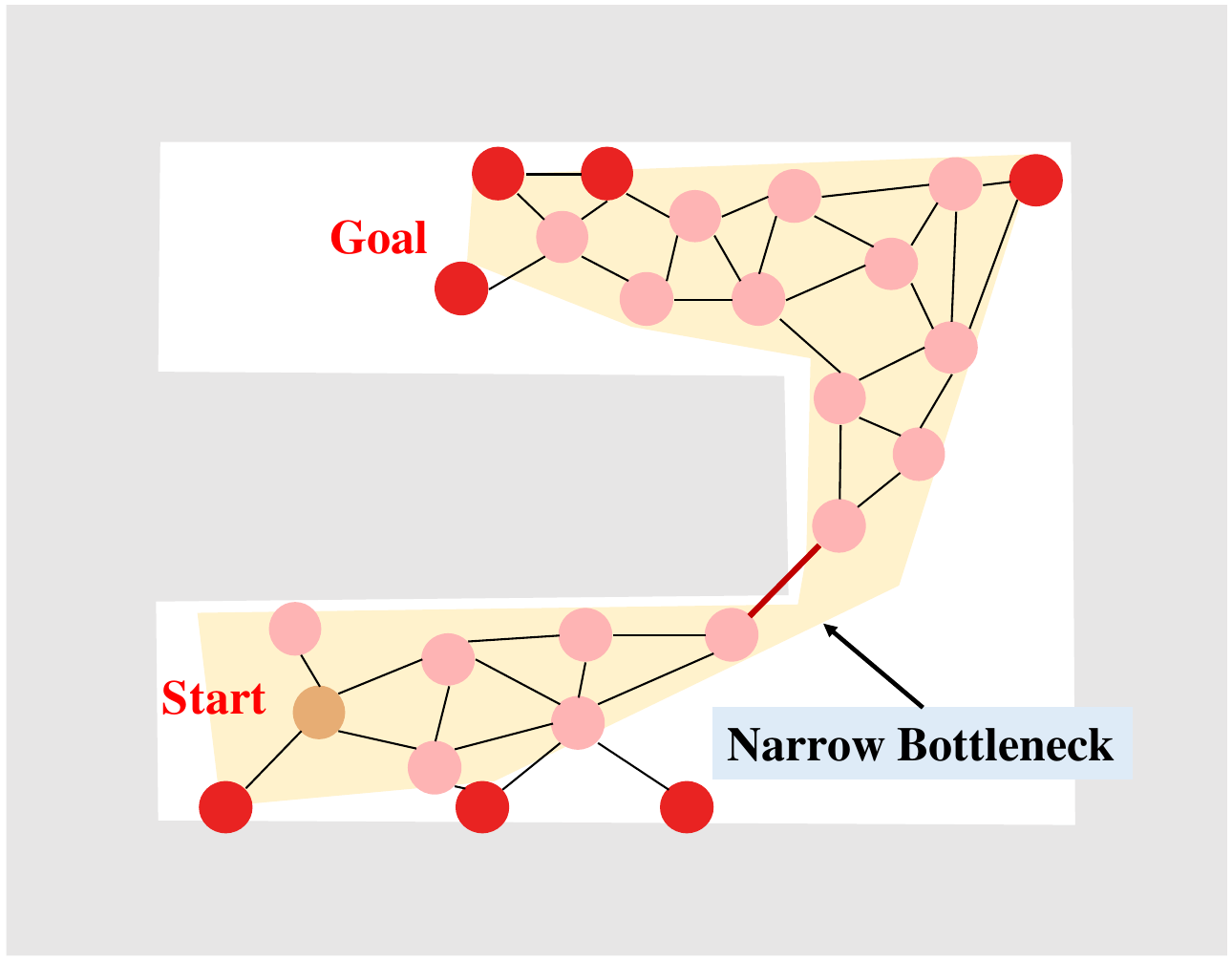}}
\caption{\added{An illustrative example of the narrow bottleneck problem induced by high-return (HR) sampling. The left panel shows the high-reward region (yellow), where a narrow bottleneck forms near the first corner. The right panel depicts that landmarks are sparse near the bottleneck, which affects the quality of the landmark-based graph.}}
\label{fig:narrow_bottleneck}
\end{figure*}
\added{As shown in Fig.~\ref{fig:compare_others}, performance degradation is observed during the later stage of training. Here, we provide additional analysis to investigate the underlying cause: the narrow bottleneck problem potentially induced by high-return (HR) sampling. Here, we visualize the HR sampling behavior at 0.7M steps in the Large Ant Maze environment, as illustrated in Fig.~\ref{fig:narrow_bottleneck}. The left panel highlights the region composed of high-reward trajectories in yellow, where a narrow bottleneck emerges near the first corner. The right panel presents the landmark-based graph within this high-reward region. Due to the use of farthest point sampling (FPS) for landmark selection, the sparsity of landmarks around the bottleneck may lead to undesirable behaviors, such as colliding with walls or the movable block in Ant Push, ultimately resulting in performance degradation. This issue can be mitigated by dynamically adjusting the temperature parameter $\alpha$, for instance, by increasing $\alpha$ during the later stages of training to better balance high-return sampling and uniform sampling.}

\bibliographystyle{plainnat}
\bibliography{neurips_2023}

\end{document}